\documentclass{article}

\usepackage[main, final]{neurips_2026}

\usepackage[utf8]{inputenc}
\usepackage[T1]{fontenc}
\usepackage{hyperref}
\usepackage{url}
\usepackage{booktabs}
\usepackage{amsmath,amssymb,amsfonts,amsthm}
\usepackage{nicefrac}
\usepackage{microtype}
\usepackage{xcolor}
\usepackage{graphicx}
\usepackage{algorithm}
\usepackage{algpseudocode}
\usepackage{mdframed}

\newmdenv[
  backgroundcolor=black!4,
  linecolor=black!25,
  linewidth=0.4pt,
  innerleftmargin=8pt,
  innerrightmargin=8pt,
  innertopmargin=6pt,
  innerbottommargin=6pt,
  skipabove=6pt,
  skipbelow=6pt
]{takeaway}

\hypersetup{
    colorlinks=true,
    linkcolor=blue!60!black,
    citecolor=blue!60!black,
    urlcolor=blue!60!black,
}

\newcommand{\reals}{\mathbb{R}}

\newcommand{\F}{\mathcal{F}}

\newcommand{\Lag}{\mathcal{L}}
\newcommand{\tp}{^\top}

\newtheorem{proposition}{Proposition}

\newtheorem{corollary}{Corollary}
\theoremstyle{remark}

\title{Augmented Lagrangian Predictive Coding}

\author{%
  Jeffrey Seely \\
  Sakana AI \\
  \texttt{jeffrey@sakana.ai} \\
  \And
  Julian Gould \\
  Sakana AI \\
  \texttt{juliangould@sakana.ai} \\
}

\begin{document}
\maketitle

\begin{abstract}
Predictive coding (PC) is a local-learning alternative to backpropagation (BP), training deep networks via local energy-minimization dynamics rather than a global backward pass. We introduce Augmented Lagrangian Predictive Coding (PC-ALM), which maintains PC's inference budget but aligns each weight update toward BP by accumulating per-layer constraint errors into a layer-local Lagrange multiplier. In linear PC networks, PC-ALM converges to an equilibrium with exact BP gradients distributed across the network via only layer-local updates. We analyze PC-ALM in nonlinear PC networks up to depth 128 and show that it matches BP performance across all width-depth regimes, notably in deep narrow networks where PC underperforms. PC-ALM introduces recurrent dynamics in each layer's activations. Compared to PC's heat flow on a scalar energy, PC-ALM dynamics are driven by dual ascent on the augmented Lagrangian. We observe ``ballistic'' credit propagation across very deep networks, with credit signals evenly distributed across layers, compared to PC's slow, diffusive credit propagation. Beyond the algorithm itself, the augmented Lagrangian framework offers a generalization of PC, and may yield insights into how distributed systems could compute and propagate BP-like credit signals through purely local dynamics.

\end{abstract}

\section{Introduction}
\label{sec:intro}

Modern deep learning relies on backpropagation (BP). Yet biological circuits cannot implement BP, at least not exactly \citep{lillicrap2020backpropagation}. One reason is that BP is a \emph{global} algorithm: even though BP's chain rule spatially localizes updates, the network must still create error variables in parallel with the matching forward activations, then gate weight updates in a precisely orchestrated sequence of steps.
This gap has motivated research into \emph{local} alternatives to BP, where state, error, and weight variables evolve as a single dynamical system that requires only layer-local coupling.
Predictive coding (PC) is one of the more actively researched layer-local alternatives to BP \citep{whittington2017approximation, millidge2021predictive, salvatori2025survey}.
In a Predictive Coding Network (PCN), a mean-squared error term is introduced between each layer, the sum of which is the PC energy. Gradient descent on this energy yields a layer-local dynamical system for updating the activation states (PC inference) and weights (PC learning). The inference equilibrium yields a different activation pattern than one would get from a forward pass, resulting in a corresponding weight update that is similar to, but not the same as, the BP gradient.

In this paper, we frame PC within a broader class of local learning methods involving the augmented Lagrangian (AL).
Dual ascent on the AL---known as the method of multipliers (MM) \citep{hestenes1969multiplier}---yields a local algorithm in certain cases and has been extensively used in decentralized approaches to network training (e.g., \citet{taylor2016admm}). We show that MM can be used as a replacement for PC's inference step, but with a substantial difference: at equilibrium, the Lagrange multipliers converge to exactly the gradient signal of the supervised loss.

This connection is not entirely new. \citet{lecun1988theoretical} observed that, at equilibrium, the multipliers of the \emph{standard} Lagrangian encode the BP adjoints. Interestingly, the augmented Lagrangian is essentially the standard Lagrangian plus the PC energy (up to the shared supervised-loss term), suggesting that there is a close relationship between PC and augmented Lagrangian methods.

\paragraph{Contribution: PC-ALM.}
The caveat of MM is that exact BP computation requires convergence to equilibrium in its inner solve, which is not guaranteed for general nonlinear networks \citep{wang2019dladmm}. However, we note that this difficulty is shared by PC. The PC literature handles this empirically, fixing a finite inference budget (the number of gradient steps on the PC energy) and assessing performance directly, even if inference is undersolved \citep{innocenti2025mupc, pinchetti2025benchmarking}. This choice is often justified by good initializations that are close to equilibrium \citep{innocenti2025mupc, pinchetti2026init}. Motivated by this, we develop a finite-inference variant of MM, called PC-ALM (for Augmented Lagrangian Method), which adds one local multiplier update to each PC inference step, so the two methods can be compared at the same inference budget.\footnote{Code is available at \url{https://github.com/SakanaAI/pc-alm}.}

Mechanistically, PC-ALM modifies PC inference with a dual variable (i.e. Lagrange multiplier) per layer that integrates prediction errors and feeds back into the local activations. After a fixed number of inference steps, the weight update acts on a composite signal of the prediction error shifted by the dual. The recurrence introduces stable oscillatory transients on an intermediate time scale between PC's fast inference and slow weight updates.

We first show that in linear PC networks, PC-ALM converges to a solution with exact BP gradients. Our core experimental finding is that, in nonlinear networks, PC-ALM closes the performance gap between PC and BP without increasing PC's inference budget. Following \citet{innocenti2026infinitelimits}, we test PC-ALM at depths up to $128$ on image classification. PC-ALM matches BP across all width-depth regimes, including the deep narrow networks where PC significantly underperforms BP.

We further observe experimentally that PC-ALM exhibits surprising dynamical properties. PC networks are known to concentrate credit signals near the supervised loss \citep{pinchetti2025benchmarking, goemaere2025epc}. In contrast, PC-ALM distributes credit near-uniformly across the network. Finally, we observe that the primal-dual dynamics of PC-ALM accelerate the credit-propagation wavefront in extremely deep networks, where standard PC stalls within practical inference budgets.
\begin{figure}[t]
\centering
\includegraphics[width=\linewidth]{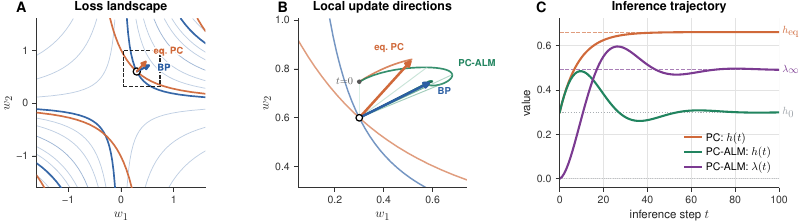}
\caption{PC-ALM credit alignment in a scalar two-layer network ($y = w_2 w_1 x$; $h=w_1 x$; $x{=}y{=}1$). \textbf{(A)} BP loss landscape with BP and equilibrated-PC iso-loss curves. \textbf{(B)} Arrows are negative weight gradients for BP, equilibrated PC, and PC-ALM at different inference steps; curves trace the PC and PC-ALM gradients over inference from $t=0$. \textbf{(C)} Inference trace. PC settles $h$ at $h_\mathrm{eq}$; PC-ALM returns $h$ to $h_0 = w_1 x$ while $\lambda$ converges to the negative of the BP adjoint $\lambda_\infty = (y - w_2 h_0)\, w_2$. Inference follows Algorithm~\ref{alg:pcalm} with $\alpha = 0.2$.}
\label{fig:pcalm-local-alignment}
\end{figure}

\section{Related Work}
\label{sec:related}

Our current work draws on an under-explored connection between the biologically plausible deep learning literature and the distributed optimization literature.

\paragraph{Predictive coding.} Predictive coding originated in computational neuroscience as a generative (or top-down) model of cortical inference, in which each layer predicts the activity of the layer below \citep{Rao1999PC,friston2009predictive}.
\citet{whittington2017approximation} later applied the same formalism to discriminative (or bottom-up) networks, where predictions run from input to output as in a standard feedforward network.
In this sense, PC provides layer-local credit assignment through inference dynamics \citep{millidge2021predictive, salvatori2025survey}, relaxing the ``locking'' constraints imposed by BP's globally sequenced forward and backward passes \citep{jaderberg2017decoupled, millidge2022ijcai}.
While standard PC inference uses local Jacobians, relaxed variants avoid weight transport and even omit backward activation derivatives \citep{millidge2020relaxing, alonso2021tightening}.
A core theme across the PC literature concerns the limitations of PC relative to BP.
\citet{pinchetti2025benchmarking} recently established benchmark standards, showing that PC matches BP on small-to-medium image classification tasks but substantially underperforms in deeper networks.
Beyond the vanishing/exploding gradients that also affect BP, a limitation specific to PC is the slow \emph{credit propagation} during inference \citep{goemaere2025epc}.
\citet{innocenti2025mupc} showed that simple tasks can be trained with PC on 100+ layer networks using a $\mu$P-style parameterization.
\citet{innocenti2026infinitelimits} observed that PC approximates BP in the limit of networks much wider than they are deep.
Even under these parameterizations, credit in PC at a finite inference budget can remain concentrated near the output (Figure~\ref{fig:init-credit-relu}), indicating that the credit propagation of PC remains an important limitation.
Other lines of work have focused on the benefits of PC relative to BP, notably the ability to work across arbitrary computational graphs \citep{salvatori2022pc-graphs} and the fact that letting activities settle before updating weights (termed \emph{prospective configuration} by \citealt{song2024prospective}) yields a weight update that provides useful higher-order information compared to BP \citep{alonso2024optimization, innocenti2024strict-saddles}.

Our work applies most directly to discriminative PC networks.
PC-ALM locally computes BP-aligned credit signals, addressing both the slow credit propagation of standard PC observed by \citet{pinchetti2025benchmarking, goemaere2025epc} and the PC-BP gap in deep, narrow networks observed by \citet{innocenti2026infinitelimits}.
At convergence, however, PC-ALM gives up the prospective configuration of PC.
Nevertheless, adding a leak to the multipliers can be viewed as interpolating between the credit propagation benefits of PC-ALM and the prospective configuration benefits of PC, which we develop in Appendix~\ref{app:leaky-duals}.

\paragraph{Augmented Lagrangian and ADMM for network training.}
A separate community has trained networks by the alternating direction method of multipliers (ADMM) or AL methods \citep{evens2021optimal, zhou2026pisa}, often for scalability, conditioning, or distributed optimization rather than biological plausibility.
This literature is often known as ``lifted'' training, since the activations are considered optimization variables in addition to the network parameters (as in PC).
The method of auxiliary coordinates \citep{carreira2014mac} minimizes the same quadratic-penalty energy as PC, but increases the penalty weight to enforce the layer constraints.
A related quadratic-penalty formulation also underlies ProxProp \citep{frerix2018proxprop}.
ADMM (an extension of MM) has been used to train networks without gradient steps \citep{taylor2016admm}.
For smooth activations, a variant of ADMM training provably converges to a KKT point at rate $O(1/k)$ \citep{zeng2021admm}.
\citet{wang2024rnn-alm} use the AL with multipliers on every operation to train RNNs, exploiting ReLU's piecewise-linear structure to obtain closed-form block updates; applied to a feedforward MLP, their method is closely related to PC-ALM.
A set of lifted training methods \citep{askari2018lifted, li2018lpom, gu2020fenchel, zach2019contrastive, wang2023bregman, wang2025unified} replace hard layer constraints with activation-specific penalties or relaxations (proximal, Fenchel/Lagrange-relaxation, contrastive, Bregman), but do not use layer multipliers. Our work differs algorithmically from this literature in that we study undersolved inference and focus less on proving convergence guarantees; this allows direct comparison to PC under its experimental standards.

\section{Background}
\label{sec:setup}

\begin{figure}[t]
    \centering
    \includegraphics[width=\linewidth]{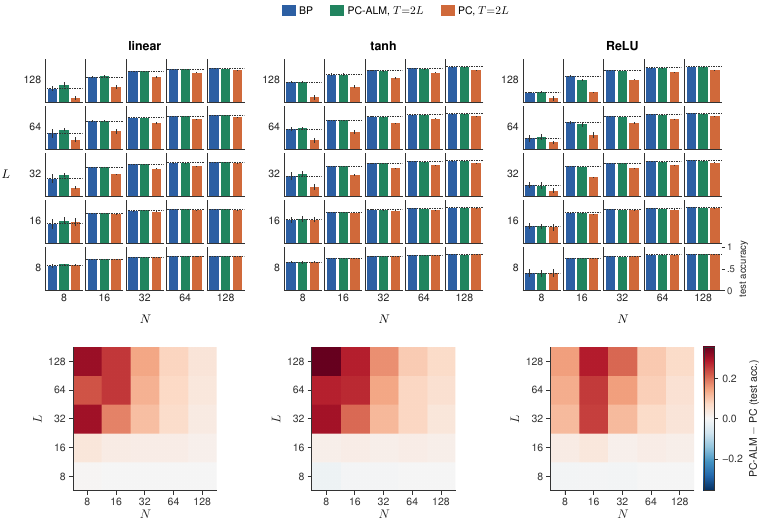}
    \caption{Width-depth results on Fashion-MNIST. PC-ALM matches BP at a budget of $T=2L$ across the full range of width $N$, depth $L$, and activations. PC-ALM outperforms PC in deep narrow networks.}
    \label{fig:nl-headline-fashion}
\end{figure}

We start by reviewing the constrained optimization formalism for training a deep feedforward network.

\paragraph{Constrained network.}
Consider a feedforward network of depth $L$ with weights $\theta = \{W_i\}_{i=1}^{L}$ and elementwise nonlinearity $\sigma$ (e.g., identity, tanh, ReLU). Given input $x \in \reals^{n_0}$ and target $y \in \reals^{n_L}$, the forward pass specifies $h_i = \sigma(W_i h_{i-1})$ from $h_0 := x$ for $i = 1, \ldots, L-1$, with linear readout $W_L h_{L-1}$. We can write the optimization problem as
\begin{equation}
\label{eq:constrained}
\begin{aligned}
\min_{\theta,\, h}\quad & \tfrac{1}{2}\bigl\|y - W_L h_{L-1}\bigr\|^2 \\
\text{s.t.}\quad & h_i = \sigma(W_i h_{i-1}), \quad i = 1, \ldots, L-1,
\end{aligned}
\end{equation}
where the activations $h = \{h_i\}_{i=1}^{L-1}$ are optimization variables alongside the weights $\theta$.

\paragraph{Predictive coding as quadratic penalty.}
The penalty relaxation of \eqref{eq:constrained} replaces each layer constraint with a squared-error term, giving the PC energy \citep{whittington2017approximation, millidge2022framework, salvatori2022pc-graphs}:
\begin{equation}
\label{eq:pc-energy}
\F_{\text{PC}}(h, \theta)
\;=\; \tfrac{1}{2}\bigl\|y - W_L h_{L-1}\bigr\|^2
\;+\; \frac{\rho}{2} \sum_{i=1}^{L-1} \bigl\|h_i - \sigma(W_i h_{i-1})\bigr\|^2,
\end{equation}
for a penalty strength $\rho$.

\paragraph{Equilibrated PC.}
For each sample $(x, y)$, PC alternates between two steps:
\begin{subequations}
\label{eq:pc-train}
\begin{align}
\text{inference:}\quad & h^\star \;=\; \arg\min_{h}\, \F_{\text{PC}}(h, \theta), \label{eq:pc-inference}\\
\text{learning:}\quad & \theta \;\gets\; \theta - \eta_\theta\, \nabla_\theta \F_{\text{PC}}(h^\star, \theta).\label{eq:pc-learning}
\end{align}
\end{subequations}
The inference step replaces a standard deep network's forward and backward passes. The alternating steps are meant to reflect fast (inference) and slow (learning) dynamics of the underlying system.

\paragraph{Finite-inference PC.}
As written, the argmin in \eqref{eq:pc-inference} globally couples all layer activations. However, gradient descent on $\F_{\text{PC}}$ decouples into layer-local updates:
\begin{equation}
\label{eq:pc-finite-inf}
h_i(t+1) \;=\; h_i(t) \;-\; \eta_h\, \nabla_{h_i}\F_{\text{PC}}(h(t), \theta),
\quad i = 1, \ldots, L-1, \quad t = 0, \ldots, T-1.
\end{equation}
Each $h_i$ appears in only two terms of $\F_{\text{PC}}$, so $\nabla_{h_i}\F_{\text{PC}}$ depends only on $h_{i-1}(t), h_i(t), h_{i+1}(t)$ and the adjacent weights $W_i, W_{i+1}$.
Activations are typically initialized from their forward-pass values. In general, the number of inference steps required to reach convergence depends on the architecture \citep{alonso2024optimization}; in practice, the iteration is typically run either to numerical convergence \citep{innocenti2024strict-saddles} or for a fixed $T$-step budget \citep{innocenti2025mupc, pinchetti2025benchmarking}. We refer to this as finite-inference PC to distinguish it from the idealized equilibrated case.

\paragraph{Augmented Lagrangian.}
Attaching multipliers $\lambda = \{\lambda_i\}_{i=1}^{L-1}$ to the constraints of \eqref{eq:constrained} and keeping the quadratic penalty gives the augmented Lagrangian \citep{hestenes1969multiplier, powell1969method, bertsekas1976multiplier}:
\begin{equation}
\label{eq:aug-lag}
\Lag_\rho(h, \theta, \lambda)
\;=\; \tfrac{1}{2}\bigl\|y - W_L h_{L-1}\bigr\|^2
\;+\; \sum_{i=1}^{L-1} \lambda_i^\top \bigl(h_i - \sigma(W_i h_{i-1})\bigr)
\;+\; \frac{\rho}{2} \sum_{i=1}^{L-1} \bigl\|h_i - \sigma(W_i h_{i-1})\bigr\|^2.
\end{equation}
Setting $\lambda \equiv 0$ recovers the PC energy: $\F_{\text{PC}}(h, \theta) = \Lag_\rho(h, \theta, 0)$. Setting $\rho = 0$ recovers the standard Lagrangian of \eqref{eq:constrained}.

\paragraph{Method of multipliers.}
Returning to the constrained optimization problem in \eqref{eq:constrained}, we can solve for $(h, \lambda)$ by performing dual ascent on the AL (the method of multipliers \citep{hestenes1969multiplier, powell1969method, bertsekas1976multiplier}). 
For each sample, initialize $\lambda_i \equiv 0$ and alternate between primal and dual updates
\begin{subequations}
\label{eq:mom}
\begin{align}
\text{primal:}\quad & h \;\gets\; \arg\min_{h}\, \Lag_\rho(h, \theta, \lambda), \label{eq:mom-primal}\\
\text{dual:}\quad & \lambda_i \;\gets\; \lambda_i + \rho\, \bigl(h_i - \sigma(W_i h_{i-1})\bigr), \quad i = 1, \ldots, L-1, \label{eq:mom-dual}
\end{align}
\end{subequations}
for $k = 1, \ldots, K$ outer rounds or until convergence criteria are met. Under standard regularity conditions, these updates converge for sufficiently large $\rho$ \citep{bertsekas1976multiplier}. The method of multipliers here is being used to solve the ``inference'' step but on the \emph{constrained} problem \eqref{eq:constrained}, not PC's quadratic relaxation \eqref{eq:pc-energy}. Therefore, at convergence to a KKT point $(h^\star, \lambda^\star)$, the activations $h^\star$ are simply what one would get from a forward pass. However, the converged multipliers $\lambda^\star$ represent meaningful local error signals, yielding a weight update:

\begin{equation}
\label{eq:mom-learning}
\text{learning:}\quad \theta \;\gets\; \theta - \eta_\theta\, \nabla_\theta \Lag_\rho(h^\star, \theta, \lambda^\star).
\end{equation}

This corresponds exactly to a BP update, as expanded below.

\paragraph{Multipliers as backpropagation gradients.}
\citet{lecun1988theoretical} observed that at any KKT point of \eqref{eq:constrained}, the multipliers of the standard Lagrangian (\eqref{eq:aug-lag} with $\rho = 0$) encode the backpropagation adjoints. Namely, with $h$ fixed at the forward-pass values, back-substitution from the output recovers $\lambda$ in closed form, with $\lambda$ equal to the negative BP adjoint at every layer (Appendix~\ref{app:kkt-proof}).

\paragraph{Local algorithm via the augmented Lagrangian.}
This identification suggests an iterative route to the BP adjoints by dual ascent on the standard Lagrangian, which decomposes layerwise (a special case of dual decomposition). On a deep network's nonconvex Lagrangian, however, the iteration can diverge with $\lambda$ unbounded \citep{nocedal2006numerical}. The augmented Lagrangian \eqref{eq:aug-lag} stabilizes the dual update without altering the KKT solution, but at the cost of breaking the layerwise decomposition; the primal argmin \eqref{eq:mom-primal} now couples all activations globally. The optimization literature typically resolves this with ADMM \citep{boyd2011distributed}. We point out that solving the primal argmin with gradient descent, as in PC, also preserves layer-locality, since $\nabla_h \Lag_\rho$ at layer $i$ depends only on adjacent layers' activations and multipliers. Therefore, pairing PC-like primal updates with dual updates suggests a route to layer-local learning with gradients more aligned with BP than PC's.
\section{Method: PC-ALM}
\label{sec:alg}

\begin{figure}[t]
    \centering
    \includegraphics[width=\linewidth]{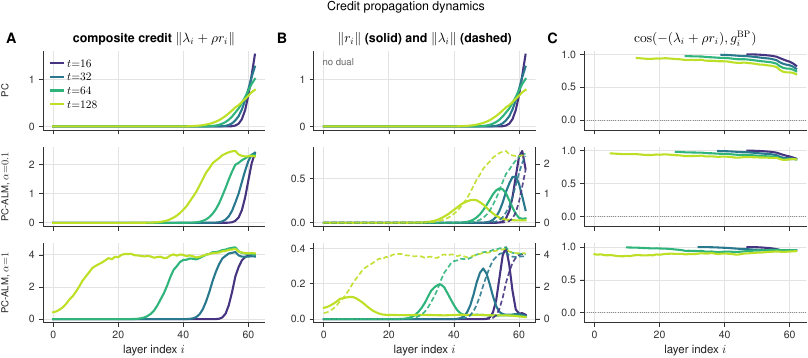}
    \caption{Diffusive versus ballistic credit propagation. ReLU residual MLP with $N = 32$, $L = 64$ at initialization, on a random Gaussian input-target pair. Panel B shows that the prediction error $r_i$ (left axis scale) leads the dual $\lambda_i$ (right axis scale) across the network. PC-ALM enables faster and more even credit propagation across the network compared to PC.}
    \label{fig:init-credit-relu}
\end{figure}
    
The method of multipliers \eqref{eq:mom} wraps an outer dual loop around the inner argmin, which must be solved at every outer round. Reaching this inner stationarity is the central bottleneck of PC inference in practice \citep{pinchetti2025benchmarking}, and is inherited by MM. We therefore develop PC-ALM, a finite-inference variant of MM that uses the same inference budget $T$ as finite-inference PC. PC-ALM replaces the argmin in \eqref{eq:mom-primal} with one gradient descent step on $\Lag_\rho$, applies the dual update immediately after, and loops $T-1$ times before a final primal step. The full procedure is Algorithm~\ref{alg:pcalm}.

\begin{algorithm}[h]
\caption{PC-ALM training}
\label{alg:pcalm}
\begin{algorithmic}[1]
\State \textbf{given} dataset $D$, penalty $\rho$, inference steps $T$, step sizes $\eta_h, \alpha, \eta_\theta$
\For{each batch $B \subset D$}
  \For{each sample $(x, y) \in B$ in parallel}
    \State $h_0 \gets x$;\quad $h_i \gets \sigma\bigl(W_i h_{i-1}\bigr)$ for $i = 1, \ldots, L-1$ \Comment{forward init}
    \State $\lambda_i \gets 0$ for $i = 1, \ldots, L-1$
    \For{$t = 1, \ldots, T-1$}
      \State $h_i \gets h_i \;-\; \eta_h\, \nabla_{h_i}\Lag_\rho\bigl(h, \theta, \lambda\bigr)$ \quad for $i = 1, \ldots, L-1$ \Comment{primal step}
      \State $\lambda_i \gets \lambda_i \;+\; \alpha\,\bigl(h_i - \sigma(W_i h_{i-1})\bigr)$ \quad for $i = 1, \ldots, L-1$ \Comment{dual step}
    \EndFor
    \State $h_i \gets h_i \;-\; \eta_h\, \nabla_{h_i}\Lag_\rho\bigl(h, \theta, \lambda\bigr)$ \quad for $i = 1, \ldots, L-1$ \Comment{final primal step}
  \EndFor
  \State $W_i \gets W_i \;-\; \eta_\theta\, \tfrac{1}{|B|} \sum_{(x,y) \in B} \nabla_{W_i}\Lag_\rho\bigl(h, \theta, \lambda\bigr)$ for $i = 1, \ldots, L$ \Comment{learning step}
\EndFor
\end{algorithmic}
\end{algorithm}

\paragraph{Inference gradients.}
To analyze the difference between $\F_{\text{PC}}$ and $\Lag_\rho$, we rewrite $\Lag_\rho$ in a more interpretable form. Let $r_i := h_i - \sigma(W_i h_{i-1})$ denote the residual (or prediction error) at layer $i$. Completing the square on each cross term in \eqref{eq:aug-lag},
\begin{equation}
\label{eq:complete-square}
\lambda_i^\top r_i \;+\; \tfrac{\rho}{2}\,\|r_i\|^2
\;=\; \tfrac{\rho}{2}\,\Bigl\|h_i - \bigl(\sigma(W_i h_{i-1}) - \lambda_i/\rho\bigr)\Bigr\|^2 \;-\; \tfrac{1}{2\rho}\,\|\lambda_i\|^2.
\end{equation}
Substituting \eqref{eq:complete-square} back into \eqref{eq:aug-lag},
\begin{equation}
\label{eq:lag-as-pc-shifted}
\Lag_\rho(h, \theta, \lambda)
\;=\; \tfrac{1}{2}\,\bigl\|y - W_L h_{L-1}\bigr\|^2
\;+\; \frac{\rho}{2}\sum_{i=1}^{L-1}\Bigl\|h_i - \bigl(\sigma(W_i h_{i-1}) - \lambda_i/\rho\bigr)\Bigr\|^2
\;-\; \frac{1}{2\rho}\sum_{i=1}^{L-1}\|\lambda_i\|^2.
\end{equation}
The last term has no $h$-dependence and drops out of $\nabla_h \Lag_\rho$. The first two terms are the PC free energy \eqref{eq:pc-energy} with each prediction target shifted from $\sigma(W_i h_{i-1})$ to $\sigma(W_i h_{i-1}) - \lambda_i/\rho$. Differentiating with respect to $h_i$ for $1 \le i \le L-2$,
\begin{align}
\label{eq:gradh-pc}
\nabla_{h_i}\F_{\text{PC}}
&\;=\; \rho\, r_i \;-\; W_{i+1}^\top\, \mathrm{diag}\bigl(\sigma'(W_{i+1} h_i)\bigr)\,\rho\, r_{i+1}, \\
\label{eq:gradh-pcalm}
\nabla_{h_i}\Lag_\rho
&\;=\; (\rho\, r_i + \lambda_i) \;-\; W_{i+1}^\top\, \mathrm{diag}\bigl(\sigma'(W_{i+1} h_i)\bigr)\,(\rho\, r_{i+1} + \lambda_{i+1}).
\end{align}
In PC-ALM, each prediction target is shifted by $-\lambda_i/\rho$ relative to PC.

\paragraph{Dual step.}
After each activity step, the multiplier accumulates the residual:
\begin{equation}
\label{eq:dual-pcalm}
\lambda_i \;\gets\; \lambda_i \;+\; \alpha\, r_i, \qquad i = 1, \ldots, L-1.
\end{equation}
The classical MM update \eqref{eq:mom-dual} uses rate $\rho$ with the activity solved to stationarity. PC-ALM takes a single inner gradient step instead, so $r_i$ carries unsettled primal gradient on top of constraint infeasibility; we use a separate hyperparameter $\alpha$ for the dual rate. $\alpha = 0$ recovers PC; $\alpha = \rho$ with the inner argmin solved recovers MM.

\paragraph{Weight update.}
After $T-1$ primal-dual cycles and a final primal step, weights are updated:
\begin{equation}
\label{eq:weight-pcalm}
\theta \;\gets\; \theta - \eta_\theta\, \nabla_\theta \Lag_\rho(h, \theta, \lambda).
\end{equation}
The composite signal $\lambda_i + \rho r_i$ at the weight gradient has $\lambda_i$ from the integrated dual history and $\rho r_i$ from the residual at the final $h$. By \eqref{eq:lag-as-pc-shifted}, this is the PC weight update with each prediction target shifted by $-\lambda_i/\rho$.

\subsection{Interpretation}
\label{sec:pcalm-interpretation}

\begin{figure}[t]
    \centering
    \includegraphics[width=\linewidth]{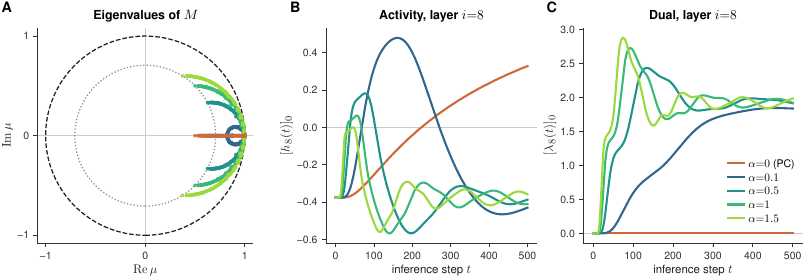}
    \caption{Linear PC-ALM inference dynamics with $L = 16$, $N = 8$. \textbf{(A)} Eigenvalues of the block iteration matrix $M$ (Equation~\eqref{eq:linear-pcalm-block}). The inner dotted circle indicates a radius of $\sqrt{1 - \eta_h \rho\, \sigma_{\max}(A)^2}$. \textbf{(B, C)} Primal and dual neuron activity traces.}
    \label{fig:linear-dynamics-spectrum}
\end{figure}

PC-ALM augments each layer with a per-layer Lagrange multiplier $\lambda_i$ of the same shape as $h_i$. On a single sample, the activity loop $t = 0, \ldots, T-1$ initializes $h$ at the forward-pass values and $\lambda \equiv 0$. The first activity step is therefore a standard PC step; with $\lambda \equiv 0$, \eqref{eq:gradh-pcalm} reduces to \eqref{eq:gradh-pc}. The first dual step then injects the layer-$i$ residual $r_i$ into $\lambda_i$. From the second activity step onward, the activity gradient at layer $i$ sees the augmented signal $\rho\, r_i + \lambda_i$ in place of $\rho\, r_i$ alone---a lagged correction entering via the integrated multiplier from prior steps.

If inference converges, the activation $h_i$ returns to its forward-pass value (the constraint $h_i = \sigma(W_i h_{i-1})$ is restored, $r_i \to 0$). Across the inference loop, $\lambda_i$ accumulates whatever pressure the layer-$i$ constraint exerts on adjacent activations, pushing $h_i$ back toward its forward-pass value, and at the fixed point this accumulated pressure is exactly the negative BP adjoint at that layer (Appendix~\ref{app:kkt-proof}).

Figure~\ref{fig:pcalm-local-alignment} illustrates this on a scalar two-layer example. At a fixed weight where equilibrated PC and BP disagree, the dual integrates to the negative BP adjoint $\lambda_\infty = (y - w_2 h_0)\, w_2$ while the activation returns to its forward-pass value $h_0 = w_1 x$. The PC-ALM weight update therefore rotates onto BP.
\section{Results}
\label{sec:results}

We now present a series of theoretical and empirical results for PC-ALM. 

\subsection{Linear PC-ALM computes exact BP gradients at equilibrium}
\label{sec:results-linear-convergence}

PC-ALM is an inexact formulation of MM since each primal step is a single gradient step instead of the full argmin. Nevertheless, we show that in a linear PC network, under stability conditions, PC-ALM converges precisely to the KKT point of the constrained optimization problem (Appendix~\ref{app:linear-pcalm-dynamics}, Proposition~\ref{prop:linear-convergence}).
Thus, in linear PCNs, PC-ALM is a layer-local dynamical system that distributes exact BP gradient signals across the network.

\subsection{Stability and oscillatory modes}
\label{sec:results-stability-oscillation}

We also provide stability conditions for PC-ALM. The linear PC-ALM dynamics are governed by an iteration matrix $M$ (Equation~\eqref{eq:linear-pcalm-block}) and stability requires that its spectral radius be less than $1$ (i.e., all eigenvalues inside the open unit disc). Dropping the supervised-loss block ($B = 0$), this reduces to the constraint-only modewise condition
\begin{equation}
\label{eq:stability-pcalm}
\underbrace{\eta_h\, \sigma_i^2\, 2\rho < 4}_{\text{PC}}
\qquad\qquad
\underbrace{\eta_h\, \sigma_i^2\, (2\rho + \alpha) < 4}_{\text{PC-ALM}}
\end{equation}
at every singular mode $\sigma_i$ of the constraint operator $A$ (Appendix~\ref{app:linear-pcalm-dynamics}, Corollary~\ref{cor:modewise-jury}). For $\alpha = 0$ we recover the stability condition for PC. The mean-field parameterizations we use (Appendix~\ref{app:architectures}) stabilize $\sigma_{\max}$ over a wide range of network depths and widths. Setting $\rho = 1$, we roughly obtain $\eta_{\max} \approx 0.5$ and $\alpha_{\max} \approx 2$.

The oscillatory modes of PC-ALM are shown in Figure~\ref{fig:linear-dynamics-spectrum}. Notably, in standard PC, all activation dynamics are gradient flows and do not produce oscillatory transients (e.g., eigenvalues on the real line in Figure~\ref{fig:linear-dynamics-spectrum} for $\alpha = 0$). The complex eigenvalues of PC-ALM lie within an annulus with inner radius $\sqrt{1 - \eta_h\, \rho\, \sigma_{\max}(A)^2}$ and outer radius $1$. The dual rate $\alpha$ slides them along this circle, setting oscillation frequency.

Importantly, convergence to the KKT solution and the BP gradients it provides does not depend on the specific stable choice of $\eta_h$ and $\alpha$; these values only shape the transient dynamics on the path to equilibrium.

\subsection{PC-ALM closes the PC-BP gap in deep narrow networks}
\label{sec:results-deep-networks}

Beyond the linear analysis, we test PC-ALM experimentally in nonlinear PCNs.
We sweep $(N, L) \in \{8, 16, 32, 64, 128\}^2$ in residual MLPs under the mean-field parameterization of \citet{innocenti2026infinitelimits} (details in Appendix~\ref{app:experimental-details}), comparing PC, PC-ALM, and BP at matched inference budget on Fashion-MNIST and MNIST with identity, tanh, and ReLU activations. PC performance degrades along the depth axis, reproducing the deep-narrow failure mode of \citet{innocenti2026infinitelimits}. PC-ALM recovers a fraction of the BP gap at a budget of $T=L$ (Figure~\ref{fig:headline-trace}) while closing the gap entirely at a budget of $T=2L$ across the full range of $N$, $L$, and activations (Figure~\ref{fig:nl-headline-fashion}).
Further experimental details are provided in Appendix~\ref{app:experimental-details}. On ResNet-18 with CIFAR-10 and Tiny ImageNet, PC-ALM improves on PC but remains below BP (Appendix~\ref{app:resnet}).

\subsection{PC-ALM is robust to reparameterization}
\label{sec:results-reparameterization}

We investigate whether the performance of PC-ALM is due to the choice of mean-field parameterization, which stabilizes both PC and BP. Holding $(N, L) = (32, 64)$, we sweep two parameterization probes around the mean-field endpoint of \citet{innocenti2026infinitelimits}: a standard parameterization to mean-field interpolation $\lambda_{\text{sp}} \in \{0, 0.25, 0.5, 0.75, 1\}$ at $\gamma_0 = 1$, and a lazy-to-rich axis $\gamma_0 \in \{0.1, 0.5, 1, 2, 3, 4\}$ at $\lambda_{\text{sp}} = 1$. Across all values, PC-ALM matches the performance of BP (Figs.~\ref{fig:param-sweep-fashion}, \ref{fig:param-sweep-mnist}), showing that it remains stable and robust even across unfavorable parameterization regimes.

\subsection{Ballistic credit propagation wavefront}
\label{sec:results-faster-credit-propagation}

In addition to the experimental verification above, we note further surprising dynamical properties of PC-ALM.
PCNs are known to exhibit signal decay across layers \citep{pinchetti2025benchmarking}, which is exponential in depth \citep{goemaere2025epc}.
We reproduce this result as shown in Figure~\ref{fig:init-credit-relu} (top). Surprisingly, the primal-dual dynamics of PC-ALM accelerate the credit propagation wavefront across the network until the composite signal $\lambda_i+\rho r_i$ distributes evenly. The prediction-error component $r_i$ propagates ahead of the lagged dual $\lambda_i$ (Figure~\ref{fig:init-credit-relu}B). Appendix~\ref{app:credit-wavefront} gives the linear wavefront calculation. We call this ballistic credit propagation in contrast to PC's diffusive propagation.

\subsection{PC-ALM's BP alignment exhibits an inflection at a predictable inference step}
\label{sec:results-credit-alignment-inflection}

We observe that PC-ALM provides a modest-but-measurable boost to PC at low inference budgets, but that this boost ``snaps'' to stronger BP alignment at inference steps of around twice the network depth when using default hyperparameter $\alpha = 1$. This phenomenon is illustrated in Figure~\ref{fig:alpha-t-two-panel}A, where the BP alignment yields an inflection point whose location is controlled by $\alpha$. Higher $\alpha$ pulls the inflection to an earlier inference step. We found this phenomenon to be present for all depths and widths. Further, the location of the inflection can be predicted across a wide range of network depths and widths, namely, at $t \approx L/\sqrt{\alpha\,\eta_h}$ (Figure~\ref{fig:alpha-t-two-panel}B), with arguments developed in Appendix~\ref{app:credit-wavefront}.

\begin{figure}[t]
    \centering
    \includegraphics[width=0.9\linewidth]{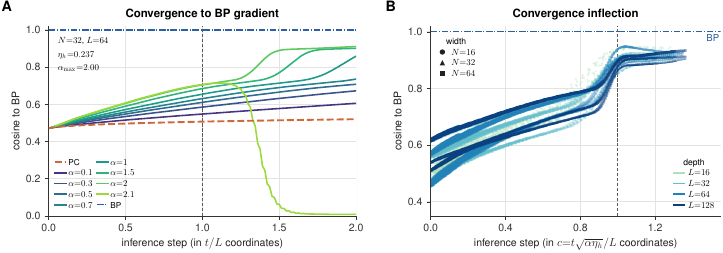}
    \caption{Cosine between the PC-ALM weight gradient and the BP weight gradient at an initialized network. \textbf{(A)} Alignment as a function of iteration, normalized to $L$. \textbf{(B)} Alignment as a function of iteration $t$, in units of $c = t\sqrt{\alpha\, \eta_h}/L$ across different depths and widths.}
    \label{fig:alpha-t-two-panel}
\end{figure}
    
\subsection{PC-ALM is inference efficient}
\label{sec:results-inference-efficiency}

From a pragmatic standpoint, PC-ALM doubles PC's activation memory cost, albeit with modest additional compute. Under mean-field scaling, \citet{innocenti2026infinitelimits} show that increasing width aligns PC's gradient closer to BP, giving two routes to BP alignment: widening PC, or running PC-ALM at fixed width. To achieve a $0.9$ cosine similarity to BP, PC-ALM is more efficient by an order of magnitude. Results are in Figure~\ref{fig:budget-to-bp}, with the argument developed in Appendix~\ref{app:complexity-comparison}.

\subsection{Dependence on initialization}
\label{sec:results-init-failure}
Both PC and PC-ALM degrade when the initial activities are displaced from the forward pass (Appendix~\ref{app:init-failure}).
PC-ALM stays ahead of PC up to a displacement of $0.3$ times the RMS activity, and falls behind only at larger displacements, where both are far below BP.
Far from the KKT point, the multipliers integrate large residuals unrelated to the output loss which can degrade performance worse than PC.

\subsection{Leaky multipliers interpolate between PC and PC-ALM}
\label{sec:results-leaky}
Adding a leak $\kappa$ to the multiplier update, $\lambda_i \gets (1-\kappa)\lambda_i + \alpha r_i$, interpolates between PC-ALM ($\kappa \to 0$) and PC ($\alpha/\kappa \to 0$) (Appendix~\ref{app:leaky-duals}).
It therefore also interpolates between the benefits of each, namely prospective configuration for PC and credit propagation and BP alignment for PC-ALM.
The resulting Pareto front bows toward the upper right, so intermediate leaks recover much of PC's prospective configuration while keeping most of PC-ALM's BP alignment (Figure~\ref{fig:leaky-duals}).

\section{Discussion}
\label{sec:discussion}

We introduce PC-ALM, a generalization of finite-inference PC that arises as finite-inference dual ascent on the augmented Lagrangian of a feedforward neural network. PC and PC-ALM address two different optimization problems---the quadratic relaxation in PC's case, and the original constrained problem in PC-ALM's. They share the same inference budget aside from PC-ALM's dual update, but exhibit distinct dynamical properties. We verified that PC-ALM is stable for linear PCNs and provides exact BP-aligned credit signals across the network through layer-local accumulation of prediction errors. In nonlinear networks, we further found that PC-ALM is robust and matches BP performance, notably in parameter regimes where PC and BP diverge. We expand on current limitations of our work in Appendix~\ref{app:limitations}.

\subsection{Connection to other Lagrangian-based methods}

PC-ALM is one of several training methods that begin by writing
feedforward training as a constrained problem $\min
\mathtt{Loss}(h_L,y)$ subject to $h_\ell = \sigma(W_\ell h_{\ell-1})$, and attaching Lagrange
multipliers and/or quadratic penalties to the layer constraints.
Methods within this family differ in (i) which terms of the augmented
Lagrangian they retain, (ii) how they solve the inner activity
subproblem, and (iii) how the multipliers, if any, are recovered.
Standard PC relies only on the penalty.

Like PC, the BlockProp algorithm \citep{gotmare2018decoupling} works with a quadratic penalty for a different motivation---model-parallel training across blocks of layers. The authors use an L2 gluing penalty and note that the Taylor-style Lagrange multiplier was unnecessary and less stable in their experiments. They further show
that BlockProp with single-layer blocks and a quadratic gluing energy
is equivalent to Equilibrium Propagation (EP) of
\citet{scellier2017equilibrium}. EP itself takes a different route around the dual-ascent
instability. Rather than iterating the multipliers, it recovers them
implicitly by \emph{nudging}---a finite-difference method on the
energy gradient between a free fixed point ($\beta=0$, in EP notation)
and a weakly-clamped fixed point ($\beta \to 0^+$). This approach
yields contrastive Hebbian updates that estimate $\lambda^\star$
without ever carrying it as state. EP therefore needs two fixed-point phases; PC-ALM runs primal and dual updates in a single loop. Dual Propagation \citep{hoier2023dual} extends the EP nudging idea to a single-phase dyadic-neuron scheme, where the per-neuron difference between nudged states plays a role structurally analogous to PC-ALM's multiplier.
The recent work of \citet{scurria2026physicaltheorybackpropagationexact} resembles PC-ALM but via the standard Lagrangian.
We expand on further related work in Appendix~\ref{app:related-work}.

\begin{ack}
We thank the Sakana AI research team for their feedback.
This work has been done at Sakana AI without any other source of funding.
The authors declare no competing interests.
\end{ack}

\bibliographystyle{plainnat}
\bibliography{references/references}

\begin{thebibliography}{61}
\providecommand{\natexlab}[1]{#1}
\providecommand{\url}[1]{\texttt{#1}}
\expandafter\ifx\csname urlstyle\endcsname\relax
  \providecommand{\doi}[1]{doi: #1}\else
  \providecommand{\doi}{doi: \begingroup \urlstyle{rm}\Url}\fi

\bibitem[Alonso and Neftci(2021)]{alonso2021tightening}
Nicholas Alonso and Emre Neftci.
\newblock Tightening the biological constraints on gradient-based predictive coding.
\newblock In \emph{International Conference on Neuromorphic Systems 2021}, pages 1--9, 2021.
\newblock \doi{10.1145/3477145.3477148}.

\bibitem[Alonso et~al.(2024)Alonso, Krichmar, and Neftci]{alonso2024optimization}
Nicholas Alonso, Jeffrey~L. Krichmar, and Emre Neftci.
\newblock Understanding and improving optimization in predictive coding networks.
\newblock In \emph{Proceedings of the AAAI Conference on Artificial Intelligence}, volume~38, pages 10812--10820, 2024.
\newblock \doi{10.1609/aaai.v38i10.28954}.

\bibitem[Askari et~al.(2018)Askari, Negiar, Sambharya, and El~Ghaoui]{askari2018lifted}
Armin Askari, Geoffrey Negiar, Rajiv Sambharya, and Laurent El~Ghaoui.
\newblock Lifted neural networks.
\newblock \emph{arXiv preprint arXiv:1805.01532}, 2018.

\bibitem[Bengio(2014)]{bengio2014autoencodersprovidecreditassignment}
Yoshua Bengio.
\newblock How auto-encoders could provide credit assignment in deep networks via target propagation.
\newblock \emph{arXiv preprint arXiv:1407.7906}, 2014.

\bibitem[Bertsekas(1976)]{bertsekas1976multiplier}
Dimitri~P. Bertsekas.
\newblock Multiplier methods: A survey.
\newblock \emph{Automatica}, 12\penalty0 (2):\penalty0 133--145, 1976.

\bibitem[Bosca and Ghrist(2026)]{bosca2026local}
Vicente Bosca and Robert Ghrist.
\newblock Neural networks as local-to-global computations.
\newblock \emph{arXiv preprint arXiv:2603.14831}, 2026.

\bibitem[Boyd et~al.(2011)Boyd, Parikh, Chu, Peleato, and Eckstein]{boyd2011distributed}
Stephen Boyd, Neal Parikh, Eric Chu, Borja Peleato, and Jonathan Eckstein.
\newblock Distributed optimization and statistical learning via the alternating direction method of multipliers.
\newblock \emph{Foundations and Trends in Machine Learning}, 3\penalty0 (1):\penalty0 1--122, 2011.
\newblock \doi{10.1561/2200000016}.

\bibitem[Bredenberg et~al.(2023)Bredenberg, Williams, Savin, Richards, and Lajoie]{bredenberg2023formalizing}
Colin Bredenberg, Ezekiel Williams, Cristina Savin, Blake Richards, and Guillaume Lajoie.
\newblock Formalizing locality for normative synaptic plasticity models.
\newblock In A.~Oh, T.~Naumann, A.~Globerson, K.~Saenko, M.~Hardt, and S.~Levine, editors, \emph{Advances in Neural Information Processing Systems}, volume~36, pages 5653--5684. Curran Associates, Inc., 2023.
\newblock URL \url{https://proceedings.neurips.cc/paper_files/paper/2023/file/120339238f293d4ae53a7167403abc4b-Paper-Conference.pdf}.

\bibitem[Carreira-Perpi{\~n}{\'a}n and Wang(2014)]{carreira2014mac}
Miguel~{\'A}. Carreira-Perpi{\~n}{\'a}n and Weiran Wang.
\newblock Distributed optimization of deeply nested systems.
\newblock In \emph{Proceedings of the 17th International Conference on Artificial Intelligence and Statistics, PMLR 33}, 2014.

\bibitem[Evens et~al.(2021)Evens, Latafat, Themelis, Suykens, and Patrinos]{evens2021optimal}
Brecht Evens, Puya Latafat, Andreas Themelis, Johan Suykens, and Panagiotis Patrinos.
\newblock Neural network training as an optimal control problem: An augmented {L}agrangian approach.
\newblock In \emph{2021 60th IEEE Conference on Decision and Control (CDC)}, pages 5136--5143. IEEE, 2021.
\newblock \doi{10.1109/CDC45484.2021.9682842}.

\bibitem[Frerix et~al.(2018)Frerix, M{\"o}llenhoff, Moeller, and Cremers]{frerix2018proxprop}
Thomas Frerix, Thomas M{\"o}llenhoff, Michael Moeller, and Daniel Cremers.
\newblock Proximal backpropagation.
\newblock In \emph{International Conference on Learning Representations}, 2018.
\newblock URL \url{https://arxiv.org/abs/1706.04638}.

\bibitem[Friston and Kiebel(2009)]{friston2009predictive}
Karl Friston and Stefan Kiebel.
\newblock Predictive coding under the free-energy principle.
\newblock \emph{Philosophical Transactions of the Royal Society B: Biological Sciences}, 364\penalty0 (1521):\penalty0 1211--1221, 2009.
\newblock \doi{10.1098/rstb.2008.0300}.

\bibitem[Goemaere et~al.(2026)Goemaere, Oliviers, Bogacz, and Demeester]{goemaere2025epc}
C\'{e}dric Goemaere, Gaspard Oliviers, Rafal Bogacz, and Thomas Demeester.
\newblock e{PC}: Fast and deep predictive coding in digital simulation.
\newblock In \emph{Proceedings of the 43rd International Conference on Machine Learning}, volume 306 of \emph{Proceedings of Machine Learning Research}, pages 35407--35442. PMLR, 2026.

\bibitem[Gotmare et~al.(2018)Gotmare, Thomas, Brea, and Jaggi]{gotmare2018decoupling}
Akhilesh Gotmare, Valentin Thomas, Johanni Brea, and Martin Jaggi.
\newblock Decoupling backpropagation using constrained optimization methods.
\newblock In \emph{ICML 2018 Workshop on Credit Assignment in Deep Learning and Deep Reinforcement Learning}, 2018.
\newblock URL \url{https://openreview.net/forum?id=BygR79WfWm}.

\bibitem[Gu et~al.(2020)Gu, Askari, and El~Ghaoui]{gu2020fenchel}
Fangda Gu, Armin Askari, and Laurent El~Ghaoui.
\newblock Fenchel lifted networks: A {L}agrange relaxation of neural network training.
\newblock In \emph{Proceedings of the 23rd International Conference on Artificial Intelligence and Statistics, PMLR 108}, 2020.

\bibitem[Hansen and Ghrist(2019)]{hansen2019distributed}
Jakob Hansen and Robert Ghrist.
\newblock Distributed optimization with sheaf homological constraints.
\newblock In \emph{2019 57th Annual Allerton Conference on Communication, Control, and Computing}, pages 766--773, 2019.
\newblock \doi{10.1109/ALLERTON.2019.8919796}.

\bibitem[He et~al.(2016)He, Zhang, Ren, and Sun]{he2016resnet}
Kaiming He, Xiangyu Zhang, Shaoqing Ren, and Jian Sun.
\newblock Deep residual learning for image recognition.
\newblock In \emph{Proceedings of the IEEE Conference on Computer Vision and Pattern Recognition (CVPR)}, pages 770--778, 2016.

\bibitem[Hestenes(1969)]{hestenes1969multiplier}
Magnus~R. Hestenes.
\newblock Multiplier and gradient methods.
\newblock \emph{Journal of Optimization Theory and Applications}, 4:\penalty0 303--320, 1969.

\bibitem[H{\o}ier et~al.(2023)H{\o}ier, Staudt, and Zach]{hoier2023dual}
Rasmus H{\o}ier, D.~Staudt, and Christopher Zach.
\newblock Dual propagation: Accelerating contrastive {H}ebbian learning with dyadic neurons.
\newblock In \emph{International Conference on Machine Learning, PMLR 202}, 2023.

\bibitem[Innocenti et~al.(2024)Innocenti, Achour, Singh, and Buckley]{innocenti2024strict-saddles}
Francesco Innocenti, El~Mehdi Achour, Ryan Singh, and Christopher~L. Buckley.
\newblock Only strict saddles in the energy landscape of predictive coding networks?
\newblock \emph{arXiv preprint arXiv:2408.11979}, 2024.

\bibitem[Innocenti et~al.(2025)Innocenti, Achour, and Buckley]{innocenti2025mupc}
Francesco Innocenti, El~Mehdi Achour, and Christopher~L. Buckley.
\newblock {$\mu$}{PC}: Scaling predictive coding to 100+ layer networks.
\newblock \emph{arXiv preprint arXiv:2505.13124}, 2025.

\bibitem[Innocenti et~al.(2026)Innocenti, Achour, and Bogacz]{innocenti2026infinitelimits}
Francesco Innocenti, El~Mehdi Achour, and Rafal Bogacz.
\newblock On the infinite width and depth limits of predictive coding networks.
\newblock \emph{arXiv preprint arXiv:2602.07697}, 2026.

\bibitem[Jaderberg et~al.(2017)Jaderberg, Czarnecki, Osindero, Vinyals, Graves, Silver, and Kavukcuoglu]{jaderberg2017decoupled}
Max Jaderberg, Wojciech~Marian Czarnecki, Simon Osindero, Oriol Vinyals, Alex Graves, David Silver, and Koray Kavukcuoglu.
\newblock Decoupled neural interfaces using synthetic gradients.
\newblock In \emph{Proceedings of the 34th International Conference on Machine Learning}, volume~70 of \emph{Proceedings of Machine Learning Research}, pages 1627--1635, 2017.

\bibitem[Koshal et~al.(2011)Koshal, Nedi{\'c}, and Shanbhag]{koshal2011multiuser}
Jayash Koshal, Angelia Nedi{\'c}, and Uday~V. Shanbhag.
\newblock Multiuser optimization: Distributed algorithms and error analysis.
\newblock \emph{SIAM Journal on Optimization}, 21\penalty0 (3):\penalty0 1046--1081, 2011.
\newblock \doi{10.1137/090770102}.

\bibitem[Krizhevsky(2009)]{krizhevsky2009cifar}
Alex Krizhevsky.
\newblock Learning multiple layers of features from tiny images.
\newblock Technical report, University of Toronto, 2009.

\bibitem[Le and Yang(2015)]{le2015tiny}
Ya~Le and Xuan Yang.
\newblock Tiny {ImageNet} visual recognition challenge.
\newblock Technical report, Stanford University, CS 231N, 2015.

\bibitem[LeCun(1988)]{lecun1988theoretical}
Yann LeCun.
\newblock A theoretical framework for back-propagation.
\newblock In \emph{Proceedings of the 1988 Connectionist Models Summer School}, pages 21--28, 1988.

\bibitem[LeCun et~al.(1998)LeCun, Bottou, Bengio, and Haffner]{lecun1998gradient}
Yann LeCun, L{\'e}on Bottou, Yoshua Bengio, and Patrick Haffner.
\newblock Gradient-based learning applied to document recognition.
\newblock \emph{Proceedings of the IEEE}, 86\penalty0 (11):\penalty0 2278--2324, 1998.

\bibitem[Lee et~al.(2015)Lee, Zhang, Fischer, and Bengio]{lee2015difference}
Dong-Hyun Lee, Saizheng Zhang, Asja Fischer, and Yoshua Bengio.
\newblock Difference target propagation.
\newblock In \emph{Joint European Conference on Machine Learning and Knowledge Discovery in Databases}, pages 498--515. Springer, 2015.

\bibitem[Li et~al.(2018)Li, Fang, and Lin]{li2018lpom}
Jia Li, Cong Fang, and Zhouchen Lin.
\newblock Lifted proximal operator machines.
\newblock \emph{arXiv preprint arXiv:1811.01501}, 2018.

\bibitem[Lillicrap et~al.(2020)Lillicrap, Santoro, Marris, Akerman, and Hinton]{lillicrap2020backpropagation}
Timothy~P. Lillicrap, Adam Santoro, Luke Marris, Colin~J. Akerman, and Geoffrey Hinton.
\newblock Backpropagation and the brain.
\newblock \emph{Nature Reviews Neuroscience}, 21\penalty0 (6):\penalty0 335--346, 2020.
\newblock \doi{10.1038/s41583-020-0277-3}.

\bibitem[Millidge et~al.(2020)Millidge, Tschantz, Seth, and Buckley]{millidge2020relaxing}
Beren Millidge, Alexander Tschantz, Anil Seth, and Christopher~L. Buckley.
\newblock Relaxing the constraints on predictive coding models.
\newblock \emph{arXiv preprint arXiv:2010.01047}, 2020.

\bibitem[Millidge et~al.(2021)Millidge, Seth, and Buckley]{millidge2021predictive}
Beren Millidge, Anil Seth, and Christopher~L. Buckley.
\newblock Predictive coding: A theoretical and experimental review.
\newblock \emph{arXiv preprint arXiv:2107.12979}, 2021.

\bibitem[Millidge et~al.(2022{\natexlab{a}})Millidge, Salvatori, Song, Bogacz, and Lukasiewicz]{millidge2022ijcai}
Beren Millidge, Tommaso Salvatori, Yuhang Song, Rafal Bogacz, and Thomas Lukasiewicz.
\newblock Predictive coding: Towards a future of deep learning beyond backpropagation?
\newblock In \emph{Proceedings of the Thirty-First International Joint Conference on Artificial Intelligence}, pages 5538--5545, 2022{\natexlab{a}}.

\bibitem[Millidge et~al.(2022{\natexlab{b}})Millidge, Song, Salvatori, Lukasiewicz, and Bogacz]{millidge2022framework}
Beren Millidge, Yuhang Song, Tommaso Salvatori, Thomas Lukasiewicz, and Rafal Bogacz.
\newblock A theoretical framework for inference and learning in predictive coding networks.
\newblock \emph{arXiv preprint arXiv:2207.12316}, 2022{\natexlab{b}}.

\bibitem[Millidge et~al.(2022{\natexlab{c}})Millidge, Tschantz, and Buckley]{millidge2022backprop}
Beren Millidge, Alexander Tschantz, and Christopher~L. Buckley.
\newblock Predictive coding approximates backprop along arbitrary computation graphs.
\newblock \emph{Neural Computation}, 34\penalty0 (6):\penalty0 1329--1368, 2022{\natexlab{c}}.
\newblock \doi{10.1162/neco_a_01497}.

\bibitem[Nocedal and Wright(2006)]{nocedal2006numerical}
Jorge Nocedal and Stephen~J. Wright.
\newblock \emph{Numerical Optimization}.
\newblock Springer, 2nd edition, 2006.

\bibitem[Oliviers et~al.(2026)Oliviers, Lominadze, and Bogacz]{oliviers2026sample}
Gaspard Oliviers, Elene Lominadze, and Rafal Bogacz.
\newblock Understanding sample efficiency in predictive coding.
\newblock \emph{arXiv preprint arXiv:2605.11911}, 2026.

\bibitem[Pinchetti et~al.(2025)Pinchetti, Qi, Lokshyn, Oliviers, Emde, Tang, M'Charrak, Frieder, Menzat, Bogacz, Lukasiewicz, and Salvatori]{pinchetti2025benchmarking}
Luca Pinchetti, Chang Qi, Oleh Lokshyn, Gaspard Oliviers, Cornelius Emde, Mufeng Tang, Amine M'Charrak, Simon Frieder, Bayar Menzat, Rafal Bogacz, Thomas Lukasiewicz, and Tommaso Salvatori.
\newblock Benchmarking predictive coding networks -- made simple.
\newblock \emph{arXiv preprint arXiv:2407.01163}, 2025.

\bibitem[Pinchetti et~al.(2026)Pinchetti, Frieder, Lukasiewicz, and Salvatori]{pinchetti2026init}
Luca Pinchetti, Simon Frieder, Thomas Lukasiewicz, and Tommaso Salvatori.
\newblock Faster predictive coding networks via better initialization.
\newblock \emph{arXiv preprint arXiv:2601.20895}, 2026.

\bibitem[Powell(1969)]{powell1969method}
Michael J.~D. Powell.
\newblock A method for nonlinear constraints in minimization problems.
\newblock In R.~Fletcher, editor, \emph{Optimization}, pages 283--298. Academic Press, 1969.

\bibitem[Rao and Ballard(1999)]{Rao1999PC}
Rajesh~P. Rao and Dana~H. Ballard.
\newblock Predictive coding in the visual cortex: a functional interpretation of some extra-classical receptive-field effects.
\newblock \emph{Nature Neuroscience}, 2:\penalty0 79--87, 1999.
\newblock \doi{10.1038/4580}.

\bibitem[Rosenbaum(2022)]{rosenbaum2022relationship}
Robert Rosenbaum.
\newblock On the relationship between predictive coding and backpropagation.
\newblock \emph{PLoS ONE}, 17\penalty0 (3):\penalty0 e0266102, 2022.
\newblock \doi{10.1371/journal.pone.0266102}.

\bibitem[Salvatori et~al.(2022)Salvatori, Pinchetti, Millidge, Song, Bao, Bogacz, and Lukasiewicz]{salvatori2022pc-graphs}
Tommaso Salvatori, Luca Pinchetti, Beren Millidge, Yuhang Song, Tianyi Bao, Rafal Bogacz, and Thomas Lukasiewicz.
\newblock Learning on arbitrary graph topologies via predictive coding.
\newblock In \emph{Advances in Neural Information Processing Systems}, 2022.

\bibitem[Salvatori et~al.(2026)Salvatori, Mali, Buckley, Lukasiewicz, Rao, Friston, and Ororbia]{salvatori2025survey}
Tommaso Salvatori, Ankur Mali, Christopher~L. Buckley, Thomas Lukasiewicz, Rajesh P.~N. Rao, Karl Friston, and Alexander Ororbia.
\newblock A survey on neuro-mimetic deep learning via predictive coding.
\newblock \emph{Neural Networks}, 195:\penalty0 108161, 2026.
\newblock \doi{10.1016/j.neunet.2025.108161}.

\bibitem[Scellier and Bengio(2017)]{scellier2017equilibrium}
Benjamin Scellier and Yoshua Bengio.
\newblock Equilibrium propagation: Bridging the gap between energy-based models and backpropagation.
\newblock \emph{Frontiers in Computational Neuroscience}, 11:\penalty0 24, 2017.
\newblock \doi{10.3389/fncom.2017.00024}.

\bibitem[Scurria(2026)]{scurria2026physicaltheorybackpropagationexact}
Antonino~Emanuele Scurria.
\newblock A physical theory of backpropagation: Exact gradients from the least-action principle.
\newblock \emph{arXiv preprint arXiv:2602.02281}, 2026.

\bibitem[Seely(2025)]{seely2025sheaf}
Jeffrey Seely.
\newblock Sheaf cohomology of linear predictive coding networks.
\newblock \emph{arXiv preprint arXiv:2511.11092}, 2025.

\bibitem[Song et~al.(2024)Song, Millidge, Salvatori, Lukasiewicz, Xu, and Bogacz]{song2024prospective}
Yuhang Song, Beren Millidge, Tommaso Salvatori, Thomas Lukasiewicz, Zhenghua Xu, and Rafal Bogacz.
\newblock Inferring neural activity before plasticity as a foundation for learning beyond backpropagation.
\newblock \emph{Nature Neuroscience}, 2024.
\newblock \doi{10.1038/s41593-023-01514-1}.

\bibitem[Tang et~al.(2018)Tang, Dall'Anese, Bernstein, and Low]{tang2018running}
Yujie Tang, Emiliano Dall'Anese, Andrey Bernstein, and Steven Low.
\newblock Running primal-dual gradient method for time-varying nonconvex problems.
\newblock \emph{arXiv preprint arXiv:1812.00613}, 2018.

\bibitem[Taylor et~al.(2016)Taylor, Burmeister, Xu, Singh, Patel, and Goldstein]{taylor2016admm}
Gavin Taylor, Ryan Burmeister, Zheng Xu, Bharat Singh, Ankit Patel, and Tom Goldstein.
\newblock Training neural networks without gradients: A scalable {ADMM} approach.
\newblock In \emph{Proceedings of the 33rd International Conference on Machine Learning, PMLR 48}, 2016.

\bibitem[Wang et~al.(2019)Wang, Yu, Chen, and Zhao]{wang2019dladmm}
Junxiang Wang, Fuxun Yu, Xiang Chen, and Liang Zhao.
\newblock {ADMM} for efficient deep learning with global convergence.
\newblock In \emph{Proceedings of the 25th ACM SIGKDD International Conference on Knowledge Discovery \& Data Mining}, pages 111--119, 2019.
\newblock \doi{10.1145/3292500.3330936}.

\bibitem[Wang and Benning(2023)]{wang2023bregman}
Xiaoyu Wang and Martin Benning.
\newblock Lifted {B}regman training of neural networks.
\newblock \emph{Journal of Machine Learning Research}, 24, 2023.

\bibitem[Wang et~al.(2025{\natexlab{a}})Wang, Valavanis, Mahmood, Mang, Benning, and Repetti]{wang2025unified}
Xiaoyu Wang, Alexandra Valavanis, Azhir Mahmood, Andreas Mang, Martin Benning, and Audrey Repetti.
\newblock A unified framework for lifted training and inversion approaches.
\newblock \emph{arXiv preprint arXiv:2510.09796}, 2025{\natexlab{a}}.

\bibitem[Wang et~al.(2025{\natexlab{b}})Wang, Zhang, and Chen]{wang2024rnn-alm}
Yue Wang, Chao Zhang, and Xiaojun Chen.
\newblock An augmented {L}agrangian method for training recurrent neural networks.
\newblock \emph{SIAM Journal on Scientific Computing}, 47\penalty0 (1):\penalty0 C22--C51, 2025{\natexlab{b}}.
\newblock \doi{10.1137/23M1627614}.

\bibitem[Whittington and Bogacz(2017)]{whittington2017approximation}
James C.~R. Whittington and Rafal Bogacz.
\newblock An approximation of the error backpropagation algorithm in a predictive coding network with local {H}ebbian synaptic plasticity.
\newblock \emph{Neural Computation}, 29\penalty0 (5):\penalty0 1229--1262, 2017.
\newblock \doi{10.1162/neco_a_00949}.

\bibitem[Wu and He(2018)]{wu2018group}
Yuxin Wu and Kaiming He.
\newblock Group normalization.
\newblock In \emph{Proceedings of the European Conference on Computer Vision (ECCV)}, pages 3--19, 2018.

\bibitem[Xiao et~al.(2017)Xiao, Rasul, and Vollgraf]{xiao2017fashion}
Han Xiao, Kashif Rasul, and Roland Vollgraf.
\newblock Fashion-{MNIST}: a novel image dataset for benchmarking machine learning algorithms.
\newblock \emph{arXiv preprint arXiv:1708.07747}, 2017.

\bibitem[Zach and Estellers(2019)]{zach2019contrastive}
Christopher Zach and Virginia Estellers.
\newblock Contrastive learning for lifted networks.
\newblock \emph{arXiv preprint arXiv:1905.02507}, 2019.

\bibitem[Zeng et~al.(2021)Zeng, Lin, Yao, and Zhou]{zeng2021admm}
Jinshan Zeng, Shao-Bo Lin, Yuan Yao, and Ding-Xuan Zhou.
\newblock On {ADMM} in deep learning: Convergence and saturation-avoidance.
\newblock \emph{Journal of Machine Learning Research}, 22, 2021.

\bibitem[Zhou et~al.(2026)Zhou, Wang, Luo, Zhu, and Li]{zhou2026pisa}
Shenglong Zhou, Ouya Wang, Ziyan Luo, Yongxu Zhu, and Geoffrey~Ye Li.
\newblock Preconditioned inexact stochastic {ADMM} for deep models.
\newblock \emph{Nature Machine Intelligence}, 8\penalty0 (2):\penalty0 234--245, 2026.
\newblock \doi{10.1038/s42256-026-01182-3}.

\end{thebibliography}

\appendix
\section{KKT multipliers as backpropagation adjoints}
\label{app:kkt-proof}

For a fixed set of weights $\theta$ and one sample $(x,y)$, consider the activity-constrained inner problem induced by \eqref{eq:constrained}. At any feasible KKT point of this fixed-weight problem, the Lagrange multipliers satisfy the same reverse recursion as the backpropagation adjoints, and evaluating the weight derivative of the Lagrangian at that point gives the BP weight gradient. We recapitulate this point for the standard Lagrangian (the classical observation of \citet{lecun1988theoretical}, recorded here in our notation conventions) and then show that the quadratic augmentation does not change the endpoint.

Throughout, the backpropagation adjoint at hidden layer $i$ is the total derivative of the supervised loss after composing the downstream layers:
\begin{equation}
\delta_i \;:=\; \frac{d}{d h_i}\, \tfrac{1}{2}\bigl\|y - W_L h_{L-1}\bigr\|^2
\qquad \text{evaluated along the forward pass} \quad h_j = \sigma(W_j h_{j-1}),\ h_0 = x.
\end{equation}
By the chain rule,
\begin{equation}
\delta_{L-1} = -W_L\tp\bigl(y - W_L h_{L-1}\bigr),
\qquad
\delta_i = W_{i+1}\tp\, \mathrm{diag}\bigl(\sigma'(W_{i+1} h_i)\bigr)\, \delta_{i+1} \quad \text{for } i = 1, \ldots, L-2.
\end{equation}

\subsection{Standard Lagrangian}

Setting $\rho = 0$ in \eqref{eq:aug-lag} gives the standard Lagrangian
\begin{equation}
\Lag_0(h, \theta, \lambda)
\;=\; \tfrac{1}{2}\bigl\|y - W_L h_{L-1}\bigr\|^2
\;+\; \sum_{i=1}^{L-1} \lambda_i\tp\bigl(h_i - \sigma(W_i h_{i-1})\bigr).
\end{equation}

\begin{proposition}\label{prop:kkt-standard}
For fixed $\theta$, at any feasible KKT point of the activity-constrained inner problem, the multipliers of $\Lag_0$ satisfy $\lambda_i = -\delta_i$ for $i = 1, \ldots, L-1$. Consequently, the derivative $\nabla_{W_i} \Lag_0$ evaluated at this endpoint equals the backpropagation gradient of $\tfrac{1}{2}\|y - W_L h_{L-1}\|^2$ with respect to $W_i$.
\end{proposition}

\begin{proof}
Primal feasibility forces $h_i = \sigma(W_i h_{i-1})$ for $i = 1, \ldots, L-1$, so $h$ lies on the forward pass.

\underline{\emph{Boundary, $i = L-1$}:} The activation $h_{L-1}$ enters $\Lag_0$ via the supervised loss and the layer-$(L-1)$ residual:
\begin{equation}
\nabla_{h_{L-1}} \Lag_0 \;=\; -W_L\tp\bigl(y - W_L h_{L-1}\bigr) \;+\; \lambda_{L-1}.
\end{equation}
Stationarity gives $\lambda_{L-1} = W_L\tp(y - W_L h_{L-1}) = -\delta_{L-1}$.

\underline{\emph{Recursion, $1 \le i \le L-2$}:} The activation $h_i$ enters the layer-$i$ residual as the LHS (multiplier $\lambda_i$) and the layer-$(i+1)$ residual through $\sigma(W_{i+1} h_i)$ (multiplier $\lambda_{i+1}$):
\begin{equation}
\nabla_{h_i} \Lag_0 \;=\; \lambda_i \;-\; W_{i+1}\tp\, \mathrm{diag}\bigl(\sigma'(W_{i+1} h_i)\bigr)\, \lambda_{i+1}.
\end{equation}
Stationarity gives $\lambda_i = W_{i+1}\tp \mathrm{diag}(\sigma'(W_{i+1} h_i))\, \lambda_{i+1}$. Induction from the boundary yields $\lambda_i = -\delta_i$ for all $i$.

\underline{\emph{Weight gradient}:} For $i = 1, \ldots, L-1$, only the $i$-th layer constraint involves $W_i$:
\begin{equation}
\nabla_{W_i} \Lag_0 \;=\; -\mathrm{diag}\bigl(\sigma'(W_i h_{i-1})\bigr)\, \lambda_i\, h_{i-1}\tp \;=\; \mathrm{diag}\bigl(\sigma'(W_i h_{i-1})\bigr)\, \delta_i\, h_{i-1}\tp,
\end{equation}
which is the backpropagation gradient at layer $i$. For the readout, $\nabla_{W_L}\Lag_0 = -(y - W_L h_{L-1})\, h_{L-1}\tp$, the linear-regression gradient.
\end{proof}

The minus sign in $\lambda_i = -\delta_i$ is only a convention: it comes from writing the residual as $r_i = h_i - \sigma(W_i h_{i-1})$. With the opposite residual convention, the multipliers would have the opposite sign. The invariant statement is that the multiplier recursion carries exactly the reverse-mode adjoint information, and the induced weight derivative is the BP gradient.

\subsection{Augmented Lagrangian}

The augmented Lagrangian \eqref{eq:aug-lag} adds the quadratic penalty $\tfrac{\rho}{2}\sum_{i=1}^{L-1}\|r_i\|^2$ to $\Lag_0$, where $r_i := h_i - \sigma(W_i h_{i-1})$.

\begin{proposition}\label{prop:kkt-augmented}
For the fixed-weight activity problem, the KKT conditions and endpoint weight derivative of $\Lag_\rho$ coincide with those of $\Lag_0$. In particular, $\lambda_i = -\delta_i$ at any feasible KKT point, and $\nabla_{W_i}\Lag_\rho$ evaluated at that endpoint equals the BP weight gradient.
\end{proposition}

\begin{proof}
The augmentation $\tfrac{\rho}{2}\sum_i \|r_i\|^2$ is a smooth function of $r$ whose value and all $h$-, $W$-derivatives are linear in $r$ and therefore vanish at $r = 0$. Primal feasibility forces $r = 0$, so the stationarity conditions and weight gradient of $\Lag_\rho$ at any feasible point reduce to those of $\Lag_0$, and Proposition~\ref{prop:kkt-standard} applies.
\end{proof}

\paragraph{Remark.}
Proposition~\ref{prop:kkt-augmented} is endpoint-only: it characterizes any feasible KKT point but does not say the PC-ALM iteration reaches one. Linear-regime convergence to the KKT point is in Appendix~\ref{app:linear-pcalm-dynamics}.

\paragraph{Feedforward graph extension.}
The same calculation applies to feedforward graphs with skips. At a hidden node, stationarity collects the multiplier on every incoming constraint and subtracts the Jacobian-transpose pullback from every outgoing constraint. On a DAG, solving these equations in reverse topological order gives the usual reverse-mode accumulation. The chain proof above is the single-parent, single-child specialization.

\section{Architectures}
\label{app:architectures}

\begin{figure}[t]
    \centering
    \includegraphics[width=\linewidth]{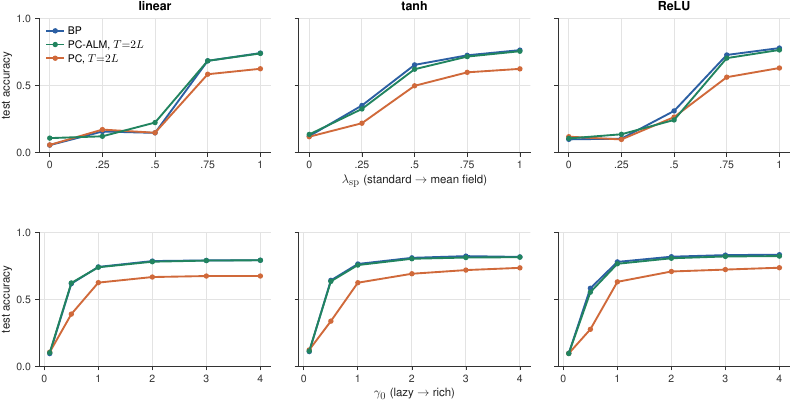}
    \caption{Parameterization sweep on Fashion-MNIST at fixed $(N, L) = (32, 64)$, single seed, $\alpha=1$. Top row: $\lambda_{\mathrm{sp}}$ interpolates from standard parameterization ($\lambda_{\mathrm{sp}}{=}0$) to mean field ($\lambda_{\mathrm{sp}}{=}1$) at $\gamma_0{=}1$. Bottom row: $\gamma_0$ traces the lazy-to-rich axis within the mean-field family at $\lambda_{\mathrm{sp}}{=}1$. PC-ALM at $T{=}2L$ tracks BP across both axes; PC at $T{=}2L$ tracks BP in the lazy regime but lags as the parameterization becomes rich.}
    \label{fig:param-sweep-fashion}
    \end{figure}
    
    \begin{figure}[t]
    \centering
    \includegraphics[width=\linewidth]{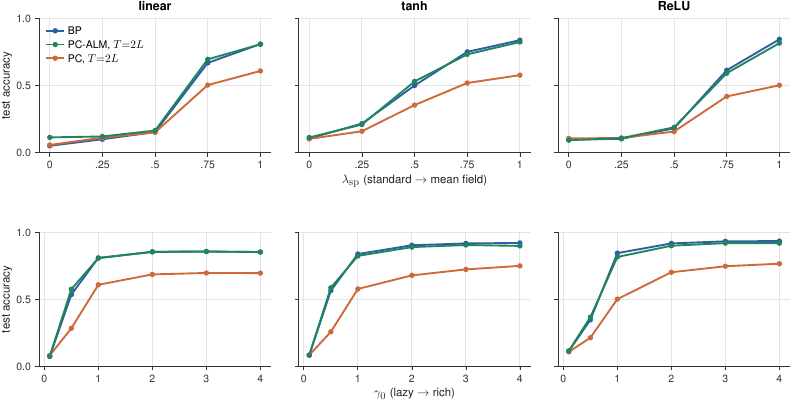}
    \caption{MNIST counterpart of Figure~\ref{fig:param-sweep-fashion}. Same axes, same fixed $(N, L) = (32, 64)$, same conclusion: PC-ALM tracks BP across the parameterization sweep; PC lags in the rich regime.}
    \label{fig:param-sweep-mnist}
\end{figure}

In our main equations we use the layer equations as $h_i = \sigma(W_i h_{i-1})$ for notational simplicity. The experiments use the scaled architectures of \citet{innocenti2026infinitelimits}, where each layer carries a scalar constant that initializes and stabilizes PC at depth and width. The architectures and constants are below.

\subsection{Architectures}

A depth-$L$ network has weights $W_1, \ldots, W_L$ and free hidden states $h_1, \ldots, h_{L-1}$ of width $N$, with $h_0 = x \in \reals^D$ clamped and prediction $\hat y \in \reals^C$. Each layer carries a scalar pre-multiplier $a_i$.

\paragraph{Chain MLP.}
\begin{equation}
h_1 \;=\; a_1\,W_1 h_0,
\qquad
h_i \;=\; a_i\,W_i\,\sigma(h_{i-1}) \quad (2 \le i \le L-1),
\qquad
\hat y \;=\; a_L\,W_L\,\sigma(h_{L-1}).
\end{equation}
The input layer is linear; interior and readout layers apply $\sigma \in \{\mathrm{id},\,\tanh,\,\mathrm{ReLU}\}$ before the linear map.

\paragraph{Residual MLP.}
\begin{equation}
h_1 \;=\; a_1\,W_1 h_0,
\qquad
h_i \;=\; h_{i-1} \;+\; a_i\,W_i\,\sigma(h_{i-1}) \quad (2 \le i \le L-1),
\qquad
\hat y \;=\; a_L\,W_L\,\sigma(h_{L-1}).
\end{equation}
Input and readout layers are unchanged from the chain; interior layers add an identity skip from the previous layer. The experiments use this topology.

\subsection{Mean-field point}

At the mean-field parameterization of \citet{innocenti2026infinitelimits}---hidden activation exponent $a = 1/2$, init exponent $b = 0$, learning-rate exponent $c = 0$, output-scale exponent $d = 1/2$, residual depth exponent $\alpha_{\mathrm{res}} = 1/2$ (distinct from the PC-ALM dual rate $\alpha$)---the per-layer pre-multipliers are
\begin{equation}
a_1 \;=\; \frac{1}{\sqrt{D}},
\qquad
a_i \;=\; \frac{1}{\sqrt{L\,N}} \quad (2 \le i \le L-1),
\qquad
a_L \;=\; \frac{1}{\gamma_0\,N},
\end{equation}
with all weights drawn $\mathcal{N}(0, 1)$ entrywise. The Adam learning rate is coupled to the parameterization \citep{innocenti2026infinitelimits},
\begin{equation}
\eta_{\mathrm{Adam}} \;=\; \eta_0\,\gamma_0^2\,\sqrt{N/L},
\end{equation}
which collapses to the $\mu$PC default \citep{innocenti2025mupc} at $\gamma_0 = 1$.

\section{Linear PC-ALM dynamics}
\label{app:linear-pcalm-dynamics}

This appendix isolates Algorithm~\ref{alg:pcalm} in the fixed-weight, linear-activation regime. With $\sigma = \mathrm{id}$ and weights frozen, one PC-ALM inner step is an affine map on the joint activity-dual state $(h, \lambda)$. We give the block iteration matrix $M$, prove convergence to the KKT point under the spectral-radius condition $\mathrm{spr}(M) < 1$, and recover the per-mode Jury bound as a corollary. Architecture-specific information (chain vs.\ skip MLP) enters only through the block constraint operator $A$; the explicit forms are in Appendix~\ref{app:architectures}.

\subsection{Setup}

Take the constrained problem \eqref{eq:constrained} with $\sigma = \mathrm{id}$ and weights $\theta$ fixed. The free hidden states are $h = \{h_i\}_{i=1}^{L-1}$, and the multipliers are $\lambda = \{\lambda_i\}_{i=1}^{L-1}$. The residuals can be stacked as a column vector,
\begin{equation}
r(h) \;=\; \begin{bmatrix} r_1 \\ \vdots \\ r_{L-1} \end{bmatrix},
\qquad
r_i \;=\; h_i - W_i\,h_{i-1},
\end{equation}
The residuals are affine in $h$ and we write $r(h) = A h + b$, where $A$ is block lower-bidiagonal with $I$ on the diagonal and the negative per-layer linear map below; $b$ absorbs the clamped-input contribution from $h_0 = x$. For a plain chain,
\begin{equation}
A_{\mathrm{chain}}
\;=\;
\begin{bmatrix}
I & & & & \\
-W_2 & I & & & \\
& -W_3 & I & & \\
& & \ddots & \ddots & \\
& & & -W_{L-1} & I
\end{bmatrix},
\qquad
b_{\mathrm{chain}}
\;=\;
\begin{bmatrix}
-W_1 x \\ 0 \\ \vdots \\ 0
\end{bmatrix}.
\end{equation}
The skip form has the same block pattern, with each sub-diagonal block
replaced by the negative skip-layer prediction map; the scaled architecture
forms are listed in Appendix~\ref{app:architectures}.

The supervised loss touches only $h_{L-1}$:
\begin{equation}
\ell(h) \;=\; \tfrac{1}{2}\,\bigl\|y - C h\bigr\|^2,
\qquad
C \;=\; \bigl[\,0\;\cdots\;0\;\;W_L\,\bigr],
\qquad
B \;:=\; C\tp C.
\end{equation}
We keep this block separate from the constraint operator: multipliers attach
to the network equations $Ah+b=0$, not to the supervised residual $y-Ch$.
The supervised term enters the activity curvature through $B$ and therefore
affects the full stability condition below.
The augmented Lagrangian \eqref{eq:aug-lag} in this regime is
\begin{equation}
\Lag_\rho(h, \lambda)
\;=\; \tfrac{1}{2}\bigl\|y - C h\bigr\|^2
\;+\; \lambda\tp\bigl(A h + b\bigr)
\;+\; \tfrac{\rho}{2}\,\bigl\|A h + b\bigr\|^2.
\end{equation}

\subsection{One PC-ALM step}

Algorithm~\ref{alg:pcalm} writes the inference loop as in-place assignments.
For the dynamical-systems analysis here, $t$ indexes the repeated inference
step, and $h(t),\lambda(t)$ denote the full stacked hidden-state and multiplier
vectors before step $t$. Unrolling one in-place primal-dual cycle gives

\begin{align}
h(t+1)
&\;=\; h(t) \;-\; \eta_h\,\bigl[\,-C\tp(y - C h(t)) \;+\; A\tp \lambda(t) \;+\; \rho\,A\tp(A h(t) + b)\,\bigr],
\label{eq:lpcalm-h} \\
\lambda(t+1)
&\;=\; \lambda(t) \;+\; \alpha\,\bigl(A h(t+1) + b\bigr).
\label{eq:lpcalm-lambda}
\end{align}
Both updates are affine in $(h, \lambda)$, so PC-ALM is an affine dynamical system on the joint state.

\subsection{Block iteration matrix}

Suppose \eqref{eq:lpcalm-h}--\eqref{eq:lpcalm-lambda} has a fixed point $(h^\star, \lambda^\star)$. In error coordinates $E(t) := h(t) - h^\star$ and $R(t) := \lambda(t) - \lambda^\star$, the affine offsets cancel, leaving the linear recursion
\begin{equation}
\label{eq:linear-pcalm-block}
\begin{bmatrix} E(t+1) \\ R(t+1) \end{bmatrix}
\;=\;
M
\begin{bmatrix} E(t) \\ R(t) \end{bmatrix},
\qquad
M \;=\;
\begin{bmatrix}
I - \eta_h(B + \rho\,A\tp A) & -\eta_h\,A\tp \\[2pt]
\alpha\,A\bigl(I - \eta_h(B + \rho\,A\tp A)\bigr) & I - \alpha\eta_h\,A A\tp
\end{bmatrix}.
\end{equation}
Chain and skip MLPs share this expression; only $A$ changes (Appendix~\ref{app:architectures}).

\subsection{Main result: convergence to the KKT point and BP gradient}

\begin{proposition}[Linear convergence of PC-ALM]\label{prop:linear-convergence}
Fix weights $\theta$ and take $\sigma = \mathrm{id}$, $\rho > 0$, $\alpha > 0$, $\eta_h > 0$. If $\mathrm{spr}(M) < 1$, then for any initialization $(h(0), \lambda(0))$ the iteration \eqref{eq:lpcalm-h}--\eqref{eq:lpcalm-lambda} converges linearly to a unique fixed point $(h^\star, \lambda^\star)$, and:
\begin{enumerate}
\item $h^\star$ satisfies $A h^\star + b = 0$ (the linear forward pass);
\item $\lambda^\star$ are the negative BP adjoints of $\tfrac{1}{2}\|y - C h\|^2$ at the linear forward pass;
\item the weight derivative at the fixed point, $\nabla_\theta \Lag_\rho(h^\star, \theta, \lambda^\star)$, equals the BP weight gradient.
\end{enumerate}
\end{proposition}

\begin{proof}
In affine coordinates the iteration takes the form $z(t+1) = M z(t) + c$ for $z = (h, \lambda)$ and a constant $c$ absorbing the input clamp $b$ and supervised-loss offset. $\mathrm{spr}(M) < 1$ implies $I - M$ is invertible, so the fixed-point equation $z^\star = M z^\star + c$ has a unique solution. In error coordinates $E(t), R(t)$ defined above, the affine offset cancels, leaving $z(t) - z^\star = M^t (z(0) - z^\star)$, which converges to $0$ at rate $\mathrm{spr}(M)$.

At the fixed point, the dual update \eqref{eq:lpcalm-lambda} forces $A h^\star + b = 0$ since $\alpha > 0$. The activity stationarity \eqref{eq:lpcalm-h} then reduces to $-C\tp(y - C h^\star) + A\tp \lambda^\star = 0$. These are the KKT conditions of the linear constrained least-squares problem
\begin{equation}
\min_h\ \tfrac{1}{2}\bigl\|y - C h\bigr\|^2 \quad \text{s.t.}\quad A h + b \;=\; 0.
\end{equation}
Appendix~\ref{app:kkt-proof} (specialized to $\sigma = \mathrm{id}$, $r_i = h_i - W_i h_{i-1}$) identifies these multipliers as the negative BP adjoints and the weight derivative at the KKT point as the BP weight gradient.
\end{proof}

\subsection{Modewise Jury bound}

To get a clean per-mode picture, drop the supervised-loss block ($B = 0$); only the constraint penalty remains in $M$. With the SVD $A = U \Sigma V\tp$ and mode coordinates $x^{(i)}(t) := v_i\tp E(t)$, $y^{(i)}(t) := u_i\tp R(t)$, each positive singular mode evolves independently under the $2 \times 2$ map
\begin{equation}
\label{eq:linear-pcalm-mode}
M_i \;=\;
\begin{bmatrix}
1 - \eta_h\rho\sigma_i^2 & -\eta_h\,\sigma_i \\[2pt]
\alpha\sigma_i\,\bigl(1 - \eta_h\rho\sigma_i^2\bigr) & 1 - \alpha\eta_h\,\sigma_i^2
\end{bmatrix},
\end{equation}
where $\sigma_i$ denotes the $i$-th singular value of $A$ (not the activation $\sigma$, which is the identity throughout this appendix).

\begin{corollary}[Schur stability per mode]\label{cor:modewise-jury}
For $\eta_h, \rho, \alpha > 0$ and $\sigma_i > 0$, $M_i$ is Schur-stable iff
\begin{equation}
\label{eq:linear-pcalm-mode-stability}
\eta_h\,\sigma_i^2\,(2\rho + \alpha) \;<\; 4.
\end{equation}
Across constraint modes, the bound binds at $\sigma_{\max}(A)$, which sets
the activity-step ceiling for the constraint-only iteration.
\end{corollary}

The eigenvalues of $M_i$ are roots of $\mu^2 - \mathrm{tr}(M_i)\,\mu + \det(M_i) = 0$ and form a complex conjugate pair when the discriminant $\mathrm{tr}(M_i)^2 - 4\det(M_i)$ is negative; Appendix~\ref{app:linear-pcalm-radius} works out the precise complex-eigenvalue regime and shows the per-mode response is then a damped oscillation rather than monotone decay. This
corollary is not the full stability test when $B\neq 0$: the exact condition
remains $\mathrm{spr}(M)<1$, and readout/output scaling enters through $B$. In the
special case where $B\,v_i = \nu_i v_i$ is diagonal in the same right singular
basis as $A$, the same Jury calculation gives
\begin{equation}
\eta_h\,\bigl(2\nu_i + (2\rho+\alpha)\sigma_i^2\bigr) < 4 .
\end{equation}
Null modes ($\sigma_i = 0$) sit outside the constraint dynamics and are
governed by any other curvature in the problem, namely the $B$ block of $M$.

\subsection{Per-mode radius and the role of \texorpdfstring{$\alpha$}{alpha}}
\label{app:linear-pcalm-radius}

The characteristic polynomial of $M_i$ has trace and determinant
\begin{equation}
\mathrm{tr}(M_i) \;=\; 2 - \eta_h\,\sigma_i^2\,(\rho + \alpha),
\qquad
\det(M_i) \;=\; 1 - \eta_h\,\rho\,\sigma_i^2.
\end{equation}
Crucially, $\det(M_i)$ does not depend on $\alpha$; only the trace does. In the
complex-eigenvalue regime $\eta_h \sigma_i^2(\rho + \alpha)^2 < 4\alpha$, the
two roots are conjugate with $|\mu_\pm|^2 = \det(M_i)$, so
\begin{equation}
|\mu_\pm| \;=\; \sqrt{\,1 - \eta_h\,\rho\,\sigma_i^2\,}.
\end{equation}
For fixed $(\eta_h, \rho, \sigma_i)$, varying $\alpha$ slides the two
eigenvalues along the circle of this radius centered at the origin in the
complex plane; only the phase changes. Different singular values of $A$
contribute different circles, so the per-mode complex eigenvalues of $M$
populate an annulus
\begin{equation}
\sqrt{\,1 - \eta_h\,\rho\,\sigma_{\max}(A)^2\,} \;\le\; |\mu| \;\le\; 1
\end{equation}
in the complex plane. Two consequences:

\begin{enumerate}
\item The spectral radius of the complex-mode eigenvalues is set by
$\eta_h, \rho$, and $\sigma_{\max}(A)$ alone. The dual rate $\alpha$ controls
phase (oscillation frequency) but not the magnitude of decay. To accelerate
contraction one must push $\eta_h\rho\sigma_{\max}^2$ closer to the Jury
ceiling, not push $\alpha$.
\item As $\alpha$ varies through the boundary $\eta_h\sigma_i^2(\rho+\alpha)^2
= 4\alpha$, eigenvalues collide on the real axis and split into a complex
conjugate pair (or vice versa). At $\alpha = 0$ the per-mode block is upper
triangular with eigenvalues $1$ and $1 - \eta_h\rho\sigma_i^2$ on the real
axis; the $1$-eigenvalue is the dual block's invariance.
\end{enumerate}

Figure~\ref{fig:linear-dynamics-spectrum} (in the main text) shows this on a small skip MLP at fixed weights: eigenvalues of $M$ populate the annulus the per-mode bound predicts, and the activity and dual traces show monotone descent at $\alpha = 0$ and damped oscillation at $\alpha > 0$.

\subsection{Reduction to penalty PC at \texorpdfstring{$\alpha = 0$}{alpha = 0}}

At $\alpha = 0$ the dual is pinned, $\lambda(t) \equiv \lambda(0)$, and $M_i$ collapses to the upper-triangular block
\begin{equation}
M_i\big|_{\alpha = 0} \;=\;
\begin{bmatrix}
1 - \eta_h\rho\sigma_i^2 & -\eta_h\,\sigma_i \\
0 & 1
\end{bmatrix}.
\end{equation}
The dual mode is invariant (eigenvalue $1$), so the joint activity--dual map
is not Schur-stable. The activity block is gradient descent on the penalty
$\tfrac{\rho}{2}\|A h + b\|^2$ along $v_i$, with the standard penalty-PC
stability bound $\eta_h\rho\sigma_i^2 < 2$. Taking the $\alpha\downarrow 0$
limit of Corollary~\ref{cor:modewise-jury} recovers the same activity-block
inequality.

\subsection{Chain vs.\ skip}

The constraint part of the dynamics is written in $A$ and $\sigma_i(A)$;
chain and skip MLPs differ through this operator (Appendix~\ref{app:architectures}).
The supervised readout contributes the separate block $B$.

\paragraph{Chain.} The off-diagonal block of $A$ is $-W_i$. The depth
products $W_{j}W_{j-1}\cdots$ appear in the inverse/Green's function of this
lower-bidiagonal system, so depth mainly shows up through the spectrum and
conditioning of $A$. Large singular values tighten the activity-step ceiling;
small singular values create slow residual modes. Both effects contribute to
the finite-$T$ inference difficulty for deep linear PC.

\paragraph{Skip.} The off-diagonal block is $-(I + s W_i)$ with
$s = 1/\sqrt{L N}$. This makes $A_{\mathrm{skip}}$ a discrete difference
operator plus a small random perturbation. Under bounded weights,
$\sigma_{\max}(A_{\mathrm{skip}})$ remains controlled, so the Jury ceiling on
$\eta_h$ need not shrink just because $L$ grows. Depth can still appear through
small modes and conditioning, but the residual parameterization removes the
most direct step-size collapse.

\section{PC-ALM credit-wave propagation}
\label{app:credit-wavefront}

In PC-ALM, each hidden-layer weight update uses the composite credit
\begin{equation}
e_i^t \;=\; \lambda_i^t \;+\; \rho\,r_i^t ,
\end{equation}
whereas PC uses only $\rho r_i^t$. At the feasible KKT point the residual vanishes and $e_i^t \to \lambda_i^t$, the negative BP adjoint (Appendix~\ref{app:kkt-proof}). The finite-$T$ question is how quickly $e_i^t$ acquires useful magnitude and BP alignment at layers far from the readout.

\paragraph{Mode dynamics.}
In the fixed-weight linear regime of Appendix~\ref{app:linear-pcalm-dynamics} with the supervised block $B$ dropped, set $q := \eta_h \sigma^2$ for one singular mode of the constraint operator $A$. The PC-ALM mode polynomial is
\begin{equation}
\label{eq:credit-wave-mode-poly}
\mu^2 \;-\; \bigl[\,2 - q(\rho + \alpha)\,\bigr]\,\mu \;+\; (1 - \rho q) \;=\; 0,
\end{equation}
stable when $q(2\rho + \alpha) < 4$ (Corollary~\ref{cor:modewise-jury}). In the complex-eigenvalue regime the two roots are $\mu_\pm = r\,e^{\pm i\theta}$, with
\begin{equation}
r \;=\; \sqrt{\,1 - \rho\,\eta_h\,\sigma^2\,},
\qquad
\cos\theta \;=\; \frac{2 - \eta_h\,\sigma^2(\rho + \alpha)}{2\sqrt{\,1 - \rho\,\eta_h\,\sigma^2\,}} .
\end{equation}
The magnitude depends on $(\rho, \eta_h, \sigma)$ alone; $\alpha$ enters only through the phase (Appendix~\ref{app:linear-pcalm-radius}). For low-frequency modes, $q \ll 1$ and
\begin{equation}
\label{eq:credit-wave-low-freq}
r \;\approx\; \exp\!\left(-\tfrac{1}{2}\,\rho\,\eta_h\,\sigma^2\right),
\qquad
\theta \;\approx\; \sqrt{\,\alpha\,\eta_h\,}\;\sigma .
\end{equation}

\paragraph{Skip-MLP dispersion.}
Under residual mean-field scaling, $A$ is a layerwise first-difference operator plus a controlled random perturbation (Appendix~\ref{app:architectures}). Its long-wavelength singular values follow $\sigma(k) \approx 2\sin(k/2) \approx k$, with $k$ the layer frequency. Substituting into \eqref{eq:credit-wave-low-freq} gives the low-$k$ dispersion relation
\begin{equation}
\label{eq:credit-wave-dispersion}
\mu_\pm(k) \;\approx\; \exp\!\left(-\tfrac{1}{2}\,\rho\,\eta_h\,k^2 \;\pm\; i\,\sqrt{\alpha\,\eta_h}\,k\right).
\end{equation}
The real exponent is diffusive damping; the imaginary exponent is the phase of a damped wave with group velocity
\begin{equation}
v_{\mathrm{credit}} \;=\; \sqrt{\,\alpha\,\eta_h\,} \qquad \text{layers per inference step.}
\end{equation}
PC has no phase term: $\mu_{\mathrm{PC}}(k) \approx \exp(-\rho\,\eta_h\,k^2)$ is pure heat diffusion. After $T$ inner steps, an output-side credit impulse therefore reaches
\begin{equation}
\label{eq:credit-reach}
R_{\mathrm{PC}}(T) \;=\; O\!\left(\sqrt{\,\rho\,\eta_h\,T\,}\right),
\qquad
R_{\mathrm{ALM}}(T) \;\approx\; T\sqrt{\,\alpha\,\eta_h\,},
\end{equation}
the latter up to a diffusive front width $O\!\bigl(\sqrt{\rho\,\eta_h\,T}\bigr)$. The strict layerwise ``light cone'' is one layer per inner step in both cases; \eqref{eq:credit-reach} measures the depth at which the local credit has crossed a usable amplitude and BP alignment, not where it first arrives.

\paragraph{Predicted inflection step.}
Equating $R_{\mathrm{ALM}}(T)$ with the network depth $L$ gives the inflection observed in Section~\ref{sec:results-credit-alignment-inflection},
\begin{equation}
\label{eq:t-infl}
t_{\mathrm{infl}} \;\approx\; \frac{L}{\sqrt{\,\alpha\,\eta_h\,}},
\qquad
\alpha_{\mathrm{reach}} \;\approx\; \frac{L^2}{\eta_h\,T^2}.
\end{equation}
At the conditioning step $\eta_h = 1/\lambda_{\max}$ with $\lambda_{\max} = \sigma_{\max}^2(A) \approx 4$ under residual mean-field, $\alpha_{\mathrm{reach}} \approx 4(L/T)^2$. For $L = 64$, $T = 2L$ this gives $\alpha_{\mathrm{reach}} \approx 1$, consistent with the inflection observed near $t \approx 2L$ at $\alpha = 1$ in the headline cell. The Jury bound \eqref{eq:linear-pcalm-mode-stability} caps $\alpha < 4/(\eta_h \lambda_{\max}) - 2\rho$; at $\eta_h = 1/\lambda_{\max}$, $\rho = 1$ this is $\alpha < 2$, so the fastest stable reach in $T$ inner steps is $T\sqrt{2/\lambda_{\max}}$. Cells with $L$ exceeding this remain budget-limited even at the cap.

\section{Leaky multipliers}
\label{app:leaky-duals}

In PC-ALM, each multiplier acts as a perfect integrator of its residual and is reset between samples.
A more biologically plausible mechanism is leaky integration, in which the multiplier decays to zero in the absence of input.
This can be implemented by adding a leak $\kappa$ on the multiplier update:
\begin{equation}
\label{eq:leaky-dual}
\lambda_i \;\gets\; (1 - \kappa)\,\lambda_i \;+\; \alpha\, r_i, \qquad 0 \le \kappa \le 1.
\end{equation}
While more biologically plausible, leaky multiplier updates also arise naturally in distributed primal-dual methods \citep{koshal2011multiuser, tang2018running}, where the leak regularizes the dual variables.

For $\kappa > 0$, a fixed point satisfies $\lambda_i = (\alpha/\kappa)\, r_i$.
The activity and weight gradients \eqref{eq:gradh-pcalm} and \eqref{eq:weight-pcalm} then equal those of PC with $\rho$ replaced by $\rho + \alpha/\kappa$.
A leaky equilibrium is therefore a PC equilibrium with a stiffer penalty.
PC-ALM is the limit $\kappa \to 0$, and PC is the limit $\alpha/\kappa \to 0$.
As discussed in Section~\ref{sec:related}, at convergence, PC-ALM gives up the so-called prospective configuration benefit of PC \citep{song2024prospective}:
in PC, the activations equilibrate at values that already move the output toward the target, making the weight update more similar to a backward Euler step \citep{alonso2024optimization}.
This admits sample efficiency benefits and alleviates interference in multi-task settings \citep{song2024prospective, oliviers2026sample}.
We measure prospective configuration by output target alignment \citep{song2024prospective, oliviers2026sample}, the cosine between the output error $y - \hat{y}$ and the first-order change in $\hat{y}$ produced by the weight update.
Adding a leak can be viewed as interpolating between the prospective configuration benefits of PC and the credit propagation and BP alignment benefits of PC-ALM.
We investigate this trade-off in Figure~\ref{fig:leaky-duals}.
Notably, the Pareto front bows toward the upper right, so intermediate leaks recover much of PC's target alignment while keeping most of PC-ALM's BP alignment.

\begin{figure}[t]
    \centering
    \includegraphics[width=\linewidth]{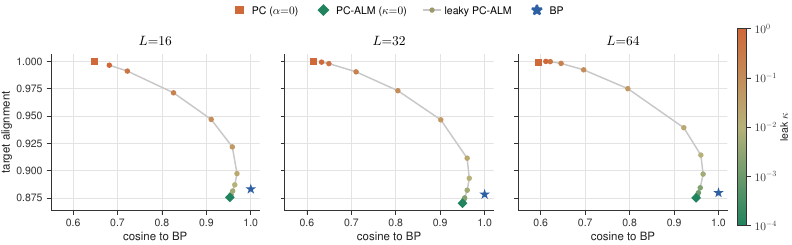}
    \caption{Target alignment against cosine to the BP gradient for leaky PC-ALM \eqref{eq:leaky-dual} with $\kappa \in [10^{-4}, 1]$ (color), PC, PC-ALM, and BP. ReLU residual MLPs with $N = 32$, $L \in \{16, 32, 64\}$ on Fashion-MNIST at initialization, batch size $1$, $T = 4L$, $\alpha = 1$, $\eta_h = 0.75/\lambda_{\max}$. Mean over $27$ paired draws of network and sample.}
    \label{fig:leaky-duals}
\end{figure}

\section{Complexity comparison}
\label{app:complexity-comparison}
 
\paragraph{Two routes to BP-aligned gradients.}
Under the mean-field parameterization of \citet{innocenti2026infinitelimits}, the
equilibrated PC energy rescaling satisfies $s(\theta) = 1 + \Theta(L/N)$
(Cor.~4.2; Eq.~19), so the architectural PC--BP mismatch scales as
$\Theta(L/N)$ at activity equilibrium. Two routes therefore reduce
the PC--BP mismatch: widening the network so that $L/N$ is small, or
running PC-ALM at fixed width. These cost, respectively, a width budget
tied to depth and an adaptive auxiliary state.
 
\paragraph{A simplifying assumption on PC's inner budget.}
We compare the two routes at a target tolerance $\varepsilon$ on the gradient
mismatch. PC's gradient mismatch has two sources: an architectural mismatch
$s(\theta) - 1 = \Theta(L/N)$ that persists even at perfect activity
equilibrium, and a finite-inference mismatch from a finite inference budget and
incomplete equilibration. We assume PC runs at a fixed activity budget
$T_{\mathrm{PC}} = \Theta(1)$, independent of $\varepsilon$, large enough that
the finite-inference mismatch is negligible beside the architectural term, so
PC's only $\varepsilon$-dependence is its width $N_{\mathrm{w}}$. We make no
matching assumption for PC-ALM, whose inner count remains
$\varepsilon$-dependent; the accounting thus favors PC.
 
\paragraph{Resource comparison.}
Per training sample, PC at width
$N_{\mathrm{w}}$ runs $T_{\mathrm{PC}}$ activity-update steps (fixed, by the
assumption above); PC-ALM at fixed width $N_0$ runs $K_\varepsilon$ joint
primal--dual iterations, in the linear fixed-weight setting of
Proposition~\ref{prop:linear-convergence}, to drive the mismatch below $\varepsilon$.
The PC-ALM convergence rate used below is the linear
fixed-weight rate of Proposition~\ref{prop:linear-convergence}; the nonlinear ReLU
experiment in Figure~\ref{fig:budget-to-bp} is an empirical diagnostic
rather than a consequence of this asymptotic argument.
 
The PC-ALM inner count follows from linear convergence. By
Proposition~\ref{prop:linear-convergence}, the joint primal--dual iteration contracts
at rate $\mathrm{spr}(M) < 1$, where $\mathrm{spr}(\cdot)$ denotes the
spectral radius. Driving the gradient mismatch below $\varepsilon$ therefore requires
\begin{equation}
  K_\varepsilon
  =
  O\left(\frac{\log(\varepsilon)}{\log \big ( \mathrm{spr}(M) \big ) }\right)
\end{equation}
joint primal--dual iterations. For PC there is no auxiliary state, and by
the fixed-budget assumption no $\varepsilon$-dependent inner count. Its
gradient mismatch is set by the architectural rescaling
$s(\theta) - 1 = \Theta(L/N)$ \citep[Eq.~19]{innocenti2026infinitelimits}, so $\varepsilon$
is controlled by the choice of width, $N_{\mathrm{w}}(\varepsilon) = \Theta(L/\varepsilon)$.
Widening pays for tighter $\varepsilon$ linearly in width, while PC-ALM pays
for tighter $\varepsilon$ logarithmically in joint iterations. The dominant
costs are summarized in Table~\ref{tab:resource-comparison}. Symbol guide: $L$ is
depth, $B$ is batch size, $N_0$ is PC-ALM's fixed width, $N_{\mathrm{w}}$ is PC's
width, $K_\varepsilon$ is the number of joint PC-ALM primal--dual iterations
to reach tolerance $\varepsilon$, $T_{\mathrm{PC}}$ is PC's fixed
($\varepsilon$-independent) activity-inference budget, $\lambda$ denotes the
ALM multiplier state, and $M$ is the linear PC-ALM iteration matrix from
Proposition~\ref{prop:linear-convergence}.
 
\begin{figure}[t]
\centering
\includegraphics[width=0.65\linewidth]{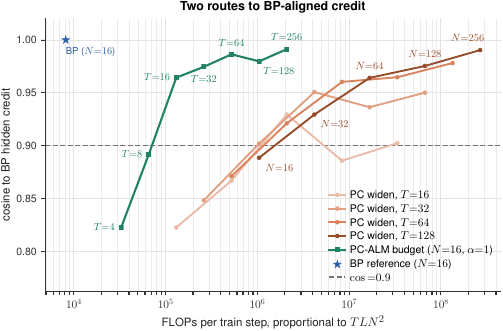}
\caption{Two routes to BP-aligned hidden credit at $L = 32$, ReLU, Fashion-MNIST. Here $T$ denotes the plotted per-sample update budget: activity steps for PC and joint primal-dual cycles for PC-ALM. Orange: PC at fixed inference budget $T$, sweeping hidden width $N$. Green: PC-ALM ($\alpha = 1$) at fixed $N = 16$, sweeping $T$. Blue star: BP reference at $N = 16$. The horizontal axis is per-step compute, proportional to $TLN^2$. In this $L=32$, ReLU/Fashion-MNIST setting, the PC-ALM branch reaches $0.9$ cosine similarity to BP at roughly an order of magnitude less compute than the PC width sweeps shown.}
\label{fig:budget-to-bp}
\end{figure}

\begin{table}[t]
  \centering
  \renewcommand{\arraystretch}{1.4}
  \begin{tabular}{l|c|c|c}
     & Parameters & Memory (activations \,+\, multipliers) & Total time \\
    \hline
    PC-ALM (fixed width $N_0$)
      & $L N_0^2$ & $2 L N_0 B$ & $K_\varepsilon\, L N_0^2 B$ \\
    PC at width $N_{\mathrm{w}}$
      & $L N_{\mathrm{w}}^2$ & $L N_{\mathrm{w}} B$ & $T_{\mathrm{PC}}\, L N_{\mathrm{w}}^2 B$ \\
      \hline
    Ratio (ALM $/$ PC)
      & $N_0^2 / N_{\mathrm{w}}^2$
      & $2 N_0 / N_{\mathrm{w}}$
      & $K_\varepsilon N_0^2 / (T_{\mathrm{PC}} N_{\mathrm{w}}^2)$ \\
  \end{tabular}
  \caption{Resource costs for PC-ALM at fixed width $N_0$ and PC at width
  $N_{\mathrm{w}}$, both under mean-field scaling at matched gradient tolerance
  $\varepsilon$, with PC's inner budget $T_{\mathrm{PC}}$ held fixed and
  $\varepsilon$-independent.}
  \label{tab:resource-comparison}
\end{table}
 
\paragraph{Space asymptotics.}
Substituting $N_{\mathrm{w}} = \Theta(L/\varepsilon)$ into the bottom row of
Table~\ref{tab:resource-comparison} gives the asymptotic ratios. Parameters scale as
$\Theta(N_0^2 \varepsilon^2 / L^2)$ and activation memory as
$\Theta(N_0 \varepsilon / L)$. Both are pure architectural ratios, free of
the inner-iteration counts; PC-ALM dominates PC on space whenever
$N_0 \varepsilon \ll L$.
 
\paragraph{Time asymptotics.}
Substituting $N_{\mathrm{w}} = \Theta(L/\varepsilon)$ and $K_\varepsilon$ asymptotics gives
\begin{equation}
  \frac{\mathrm{time}_{\mathrm{ALM}}}{\mathrm{time}_{\mathrm{PC}}}
  \;=\;
  O\!\left(
    \frac{N_0^2 \varepsilon^2}{L^2}
    \cdot
    \frac{\log(\varepsilon)}{T_{\mathrm{PC}}\, \log \big (  \mathrm{spr}(M) \big )}
  \right),
\end{equation}
for fixed $L$, $N_0$ and a stable PC-ALM iteration matrix $M$ with
$\mathrm{spr}(M) < 1$. Under the fixed-budget assumption $T_{\mathrm{PC}} =
\Theta(1)$, the ratio is
$O(\lvert\varepsilon^2 \log(\varepsilon)\rvert)$. The architectural factor $N_0^2 \varepsilon^2
/ L^2$ drives the entire ratio to zero as $\varepsilon$ tightens, provided the PC-ALM spectral gap is bounded away from zero, since
$\varepsilon^2 \log(1/\varepsilon) \to 0$. PC-ALM therefore dominates PC on time
in the small-$\varepsilon$ limit. If
$\mathrm{spr}(M) \to 1$ with depth, weights, or conditioning, the constant
can degrade. A symmetric treatment of PC's inner budget, charging PC an
$\varepsilon$-dependent inference budget as well, would cancel the
$\log(1/\varepsilon)$ factor up to the differing conditioning of the two
inference problems.
\section{Experimental details}
\label{app:experimental-details}

For our experiments we closely follow the setup of \citet{innocenti2026infinitelimits}, using their residual MLP architecture (Appendix~\ref{app:architectures}) and mean-field parameterization. The parameterization stabilizes $\sigma_{\max}(A)$, so a default $\eta_h$ would suffice in principle. However, $\sigma_{\max}^2(A)$ is sensitive in shallow networks, and $\alpha$ and $\eta_h$ are coupled. We therefore compute an empirical $\lambda_{\max} := \sigma_{\max}^2(A)$ for each $(N, L)$ and dataset, and set $\eta_h = 1/\lambda_{\max}$. This places $\alpha$ in an interpretable $[0, 2]$ range. All networks are trained for one epoch on MNIST \citep{lecun1998gradient} and Fashion-MNIST \citep{xiao2017fashion}. We additionally swept $\eta_h$ for several PC networks on a tight grid, and observed negligibly improved performance at higher $\eta_h$ followed by collapse beyond $2/\lambda_{\max}$ as expected. The $\eta_h$ sweep was unable to improve PC's performance in any meaningful way compared to PC-ALM, which we use as justification for the default $\eta_h$ choice. We use a batch size of $64$.

\begin{figure}[t]
\centering
\includegraphics[width=\linewidth]{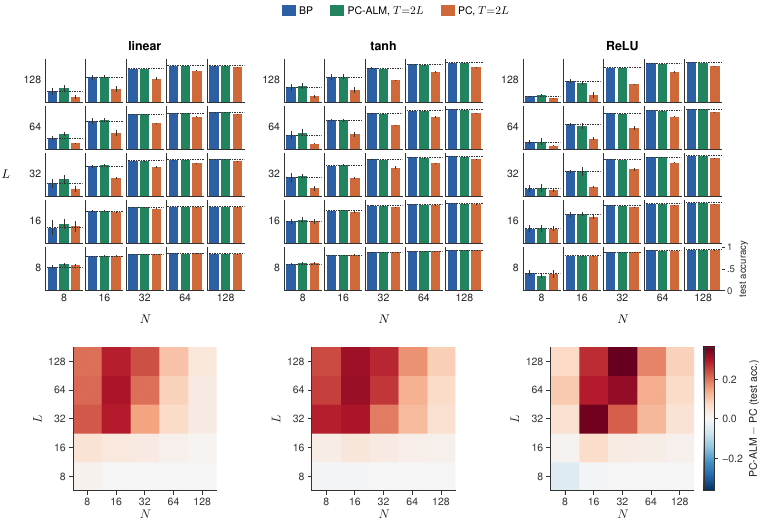}
\caption{MNIST counterpart of Figure~\ref{fig:nl-headline-fashion}. Same $(N, L)$ grid, three seeds, $\alpha=1$. Same regime structure: PC-ALM at $T=2L$ tracks BP across the grid; PC drops in deep narrow networks; the gain concentrates in deep, narrow cells.}
\label{fig:nl-headline-mnist}
\end{figure}

\begin{figure}[t]
\centering
\includegraphics[width=\linewidth]{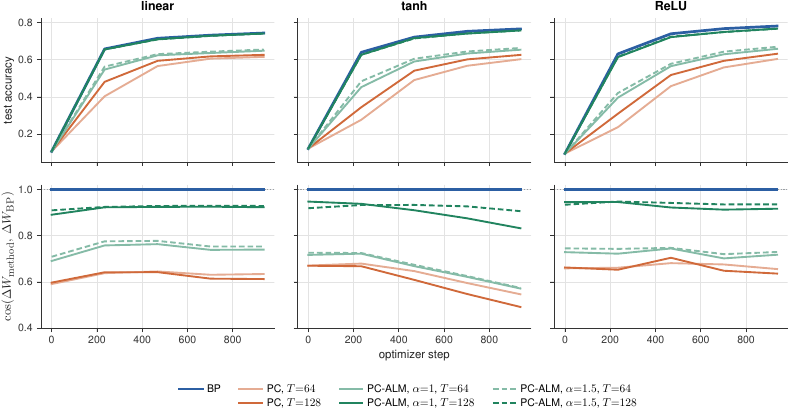}
\caption{Training curves for one epoch at $N=32$, $L=64$ on Fashion-MNIST across linear / tanh / ReLU.}
\label{fig:headline-trace}
\end{figure}

\subsection{Convolutional networks on CIFAR-10 and Tiny ImageNet}
\label{app:resnet}

We train ResNet-18 \citep{he2016resnet} on CIFAR-10 \citep{krizhevsky2009cifar} and Tiny ImageNet \citep{le2015tiny} with BP, PC, and PC-ALM (Table~\ref{tab:resnet}).
Each convolution is followed by GroupNorm with a single group \citep{wu2018group}.
We do not use BatchNorm, since it couples the examples in a minibatch.
All runs use LeakyReLU, a squared-error loss on one-hot targets, batch size $128$, $50$ epochs, and AdamW with weight decay $5\times10^{-4}$ and a warmup-cosine schedule.
PC and PC-ALM use $T=96$ inference steps.

With these settings, PC-ALM improves on PC by $18.3$ points on CIFAR-10 and by $7.7$ Top-1 points on Tiny ImageNet, and remains below BP on both.

\begin{table}[h]
\centering
\caption{ResNet-18 test accuracy (\%), mean $\pm$ s.d.\ over seeds. CIFAR-10: mean over the last five of $50$ epochs, three seeds. Tiny ImageNet: validation accuracy at epoch $50$, five seeds.}
\label{tab:resnet}
\begin{tabular}{lccc}
\toprule
Dataset & BP & PC & PC-ALM \\
\midrule
CIFAR-10 & $90.44 \pm 0.28$ & $55.93 \pm 1.85$ & $74.18 \pm 0.71$ \\
Tiny ImageNet, Top-1 & $44.93 \pm 0.19$ & $12.00 \pm 0.95$ & $19.72 \pm 0.33$ \\
Tiny ImageNet, Top-5 & $64.78 \pm 0.26$ & $31.00 \pm 1.41$ & $43.36 \pm 0.31$ \\
\bottomrule
\end{tabular}
\end{table}

\section{Displaced activity initialization}
\label{app:init-failure}

PC performance depends on the initialization of the activities \citep{pinchetti2026init}.
Typically, in PC, activities are initialized to forward-pass values, generally motivated by the fact that this activity pattern is close to inference equilibrium.
In PC-ALM, a forward initialization admits an additional justification in that the forward-pass values are the activations at the KKT point.
An issue shared by both PC and PC-ALM is how performance degrades as initialization deviates from equilibrium.
Here, we introduce a displaced initialization that allows us to test this degradation.
We test on Fashion-MNIST with a residual MLP, $(N, L) = (32, 64)$, ReLU, and $T = 2L$.
Before inference, each activation is set to:
\begin{equation}
\label{eq:displaced-init}
h_i(0) \;=\; h_i^{\mathrm{fwd}} \;+\; c\, s_i\, \xi_i, \qquad \xi_i \sim \mathcal{N}(0, I),
\end{equation}
where $h_i^{\mathrm{fwd}}$ is the forward-pass value and $s_i$ is its root-mean-square.
In PC-ALM, we initialize the multipliers to $0$.
Setting $c = 0$ recovers the initialization of the main experiments, and at $c = 1$ the noise is as large as the activity.
PC and PC-ALM receive the same draws.
Table~\ref{tab:init-failure} reports test accuracy after one epoch.

\begin{table}[h]
\centering
\caption{Test accuracy (\%) under the displaced initialization \eqref{eq:displaced-init}, mean $\pm$ s.d.\ over five seeds. Fashion-MNIST, $(N, L) = (32, 64)$, ReLU, $T = 2L$, one epoch. BP reaches $76.2$.}
\label{tab:init-failure}
\begin{tabular}{lcccccc}
\toprule
$c$ & $0$ & $0.03$ & $0.1$ & $0.3$ & $1$ & $3$ \\
\midrule
PC & $61.7 \pm 2.2$ & $57.3 \pm 3.6$ & $54.6 \pm 4.4$ & $49.9 \pm 5.1$ & $39.5 \pm 5.2$ & $25.6 \pm 8.1$ \\
PC-ALM & $75.2 \pm 0.5$ & $72.5 \pm 1.2$ & $67.3 \pm 1.3$ & $56.8 \pm 2.8$ & $24.1 \pm 5.1$ & $12.8 \pm 3.9$ \\
\bottomrule
\end{tabular}
\end{table}

In particular, we observe that both methods degrade as $c$ is increased.
Notably, we see PC-ALM outperform PC for $c$ up to $0.3$.
But at $c \ge 1$, PC beats PC-ALM, though in a regime where both networks are performing poorly on the task (test accuracy below $40\%$).

Intuitively, this flip occurs because in the high-displaced regime, PC-ALM's multipliers integrate large residuals that are unrelated to the output loss.
This emphasizes that both PC and PC-ALM are sensitive to initialization, and further that PC-ALM can fall behind PC when initialized far from the KKT point.

\section{Limitations}
\label{app:limitations}

\paragraph{Structural restrictions.}
The convergence proof for PC-ALM
(Appendix~\ref{app:linear-pcalm-dynamics},
Proposition~\ref{prop:linear-convergence}) and the credit-wave
analysis of Appendix~\ref{app:credit-wavefront} are restricted to networks
with identity activations and no bias terms. Our experiments use
activations (identity, tanh, ReLU) across the full
$(N, L) \in \{8, 16, 32, 64, 128\}^2$ grid and PC-ALM closes the PC-BP gap throughout,
but a convergence guarantee covering this setting is open.
This simplification to linear and bias-free networks
mirrors the broader PC literature: the saddle analysis of 
\citet{innocenti2024strict-saddles} likewise uses linear networks without biases for theoretical development, and nonlinear networks for empirical verification. The $\mu$PC scaling analysis of \citet{innocenti2025mupc} is likewise
linear and bias-free in its theoretical content while running nonlinear (bias-free) experiments.

\paragraph{Empirical scope.}
Experiments are restricted to MNIST, Fashion-MNIST, CIFAR-10, and Tiny ImageNet, with residual MLPs
(one epoch of training) and ResNet-18. Broader evaluation on larger and more
diverse datasets, longer training horizons, and architectures beyond
residual MLPs and ResNet-18 (attention-based) remains future work.
The width-depth grid probes the regime relevant to the PC-BP gap but is small
relative to production networks.

\paragraph{Inference budgets.}
The correspondence between the weight update and the BP gradient (Appendix~\ref{app:kkt-proof}) only holds at equilibrium. Such an equilibrium may fail to be reached with a finite inference budget, causing gradient misalignment. Moreover, the inference budget necessary to reach equilibrium (within a tolerance) grows with network depth. Far from the KKT point, for example under a displaced activity initialization, PC-ALM can fall behind PC (Section~\ref{sec:results-init-failure}).

\section{Additional Related Work}
\label{app:related-work}

There is a large body of work on biologically plausible, local learning methods. Target propagation \citep{bengio2014autoencodersprovidecreditassignment} and difference target propagation \citep{lee2015difference} have been shown to be related to predictive coding \citep{millidge2022framework}. We believe that connections of target propagation to PC-ALM in particular merit investigation. The predictive coding literature often leads with a probabilistic framework \citep{Rao1999PC,friston2009predictive,millidge2021predictive,salvatori2025survey}; the framing of PC as a quadratic penalty on the constrained optimization problem is, to our knowledge, less common (though see \citet{seely2025sheaf}). 
One benefit of the PC formalism is that it can handle non-DAG topologies as studied in \citet{salvatori2022pc-graphs}. The AL framework also extends to non-DAG architectures; at least in a linear PCN, recurrent topologies change the KKT point but do not affect convergence results. With nonlinear activations, the primal-dual dynamics can support nonlinear systems with limit cycles and other nonlinear behavior, which is not possible in standard PC's gradient-flow inference; we leave this to future investigation.
PC has also been cast in terms of sheaf theory, which offers tools for analyzing recurrent topologies \citep{seely2025sheaf}; closely related is the work of \citet{bosca2026local}, which extends the sheaf-theoretic framing to ReLU MLPs with resulting PC-like inference and learning steps. From this sheaf-theoretic perspective, the primal-dual dynamics of PC-ALM can be viewed as an application of the homological programming methods of \citet{hansen2019distributed}. Prior work on ``PC approximates backprop'' is highlighted by \citet{innocenti2026infinitelimits}; earlier work showed that with a ``fixed prediction assumption'' (a modification of standard PC), error propagation can converge to BP \citep{millidge2022backprop}, albeit in a large number of iterations; this work was further studied in \citet{rosenbaum2022relationship}. Finally, we remind the reader that while PC and PC-ALM address one aspect of BP's bio-implausibility, they do not address weight transport (the gradient of the energy still uses the transpose of the layer's weights \citep{bredenberg2023formalizing}).

\end{document}